\documentclass{article}

\usepackage{amsmath, amssymb, graphicx, url}
\usepackage{booktabs}
\usepackage{longtable}
\usepackage{mathrsfs}

\usepackage{graphicx}
\usepackage{subfigure}
\usepackage{booktabs}
\usepackage{epstopdf}
\usepackage{hyperref}
\newenvironment{claim}[1]{\par\noindent\underline{Claim:}\space#1}{}

\usepackage{algorithm}
\usepackage{algpseudocode}

\usepackage{amsthm}
\usepackage[numbers]{natbib}
\usepackage{dsfont}
\usepackage{xcolor}
\setcitestyle{authoryear,open={(},close={)}}

\newtheorem{theorem}{Theorem}
\newtheorem{lemma}[theorem]{Lemma}
\newtheorem{corollary}[theorem]{Corollary}

\usepackage{arxiv}

\usepackage[utf8]{inputenc} 
\usepackage[T1]{fontenc}    
\usepackage{hyperref}       
\usepackage{url}            
\usepackage{booktabs}       
\usepackage{amsfonts}       
\usepackage{nicefrac}       
\usepackage{microtype}      
\usepackage{lipsum}
\usepackage{graphicx}
\graphicspath{ {./images/} }

\newcommand{\mathbbm}[1]{\text{\usefont{U}{bbm}{m}{n}#1}}

\title{Learning with a Budget: Identifying the Best Arm with Resource Constraints}

\author{
  Li Zitian \\
  Department of Industrial Systems Engineering \& Management\\
  National University of Singapore\\
  Singapore \\
  \texttt{lizitian@u.nus.edu} \\
   \And
 Cheung Wang Chi \\
  Department of Industrial Systems Engineering \& Management\\
  National University of Singapore\\
  Singapore \\
  \texttt{isecwc@nus.edu.sg} \\
}

\begin{document}
\maketitle
\begin{abstract}
In many applications, evaluating the effectiveness of different alternatives comes with varying costs or resource usage. Motivated by such heterogeneity, we study the Best Arm Identification with Resource Constraints (BAIwRC) problem, where an agent seeks to identify the best alternative (aka arm) in the presence of resource constraints. Each arm pull consumes one or more types of limited resources. We make two key contributions. First, we propose the Successive Halving with Resource Rationing (SH-RR) algorithm, which integrates resource-aware allocation into the classical successive halving framework on best arm identification. The SH-RR algorithm unifies the theoretical analysis for both the stochastic and deterministic consumption settings, with a new \textit{effective consumption measure}. 
We introduce a new complexity term involving the effective complexity measure, 
and prove that SH-RR achieves a near-optimal non-asymptotic rate of convergence in the probability of identifying the optimal arm. 
Second, we uncover a fundamental difference in the convergence behaviors between the deterministic and stochastic resource consumption settings, by proving two different lower bounds. The difference illustrates the impact of uncertainty in resource consumption on the probability of identifying the best arm, confirming that the effective consumption measure is fundamental to the BAIwRC problem.
\end{abstract}


\section{Introduction}

\label{Introduction}
The best arm identification (BAI) is a central problem in 
pure exploration in the multi-armed bandit setting. The overall goal is to identify an optimal arm (an arm with the highest mean reward) through a sequence of arm pulls. Most existing works focus on the fixed confidence setting or the fixed budget setting. In the fixed confidence setting, an agent needs to output a predicted best arm after a sequence of arm pulls, with a probability of correctness above a given confidence level. The agent needs to determine when to stop dynamically. The fixed budget setting involves an upper bound on the number of arm pulls (budget). The agent aims to maximize the probability of identifying the best arm, subject to the budget constraint. 
Existing studies on BAI shed light on the relationship between the number of arm pulls and the probability of identifying an optimal arm, which provides us \emph{statistical insights} into the strategy's performance. By contrast, in the case of \emph{arm cost heterogeneity} where pulling different arms incurs different costs, the total number of arm pulls does not accurately reflect the total cost of arm pulls. In other words, existing studies, which focus on the total number of arm pulls, are yet to provide \emph{economic insights} into the total experimentation cost in cost-aware settings. 


Arm cost heterogeneity widely exists in many real-world applications, where potentially multiple types of resources are consumed for exploration. Some examples are listed below.
\begin{itemize}
    \item \textit{Advertising}. Consider a retail firm experimenting two marketing campaigns: (a) advertising on an online platform for a day, (b) providing \$5 vouchers to a selected group of recurring customers. The firm wishes to identify the more profitable one out of campaigns (a, b). The executions of (a, b) lead to different costs. A cost-aware retail firm would desire to control the total cost in experimenting these campaign choices, rather than the number of try-outs.
    \item \textit{Simulations}. The designs of transportation or queuing systems require simulation platforms to estimate the efficiency of multiple candidate designs. The simulation time consumptions usually vary across different designs. Instead of imposing a maximum number of simulations, it is more reasonable to impose a maximum on the total time consumption in all simulations conducted. 
    \item \textit{Process designs}. Process design decisions in business domains such as supply chain, service operations, pharmaceutical tests and biochemistry typically involve experimenting on a collection of alternatives, and identifying the best one in terms of profit, social welfare or any desired metric. Compared to the total number of try-outs, it is more natural to keep the total cost of experimentation in check. In applications such as pharmaceutical tests and biochemistry, each experiment consumes multiple types of resources, such as chemicals, manpower and time. The decision agent seeks to ensure all resource consumptions are within their respective budgets.
\end{itemize}

In the above settings, it is natural for the decision agent to maximize the probability of identifying the best alternative, subject to the resource constraints rather than a constraint on the maximum number of experiments / try-outs. 
What is the relationship between the total arm pulling cost and the probability of identifying the best arm? We shed light on this question in this manuscript.

\subsection{Main Contributions}

We study the problem of Best Arm Identification (BAI) under heterogeneous resource consumption across different arms. We design an algorithm, Sequential Halving with Resource Rationing (SH-RR), which takes heterogeneous resource allocation into consideration. We provide a rigorous theoretical analysis of SH-RR and establish its performance guarantees in terms of the probability of identifying the best arm. Our main contributions are summarized as follows.

First, we formulate the Best Arm Identification with Resource Constraints (BAIwRC) problem model. The proposed model considers heterogeneous resource consumption across $L$ types of budget, where $L$ is a positive integer. Different from the Fixed-Budget setting in \cite{audibert2010best}, our formulation introduces arm cost heterogeneity in a pure exploration setting. Our model allows arbitrary correlations among an arm’s random reward and its resource consumption amounts, thereby capturing a broad range of practical scenarios.

Second, we design and analyze the Successive Halving with Resource Rationing (SH-RR) algorithm. SH-RR eliminates sub-optimal arms in multiple phases, and rations adequate amounts of resources to each phase to ensure sufficient exploration in all phases. We prove that SH-RR achieves nearly best possible performance guarantees. We establish an upper bound on the SH-RR's failure probability, which is the probability of returning a sub-optimal arm at the end. Complementarily, we establish a lower bound on the failure probability of any algorithm, and the established lower bound nearly matches the upper bound for SH-RR. 
Both the upper and lower bounds involved the \textit{effective consumption measure} $f(b,\sigma, d)$ (to be defined in Section \ref{sec:main}), a novel ingredient in our analysis that quantifies how the stochasticity in resource consumption impacts the overall difficulty of the problem instance.

Thirdly, we prove lower bound results to show that our proposed algorithm SH-RR is nearly optimal. Fixing the mean value of consumption and rewards, we show that in certain cases, the stochastic consumption setting behaves similarly to the deterministic consumption setting, with the complexity terms differing only by a constant factor. However, we also identify scenarios—such as when the consumption follows a Bernoulli distribution—where the randomness in resource consumption makes the problem strictly harder. Numerical simulations in the case of two arms further support this observation. 
The fundamentally different lower bounds of the deterministic and stochastic consumption settings not only indicate the near-optimality of the algorithm SH-RR, but also shed light on further analysis on the impact of randomness in consumption.

Finally, we point out that a preliminary version of this paper appears in the 27th International Conference on Artificial Intelligence and Statistics (AISTATS 2024) \cite{li2024best}. The current manuscript provides significant additional contributions in two directions. First, we enhance the theoretical analysis of the SH-RR algorithm by unifying the treatment of failure probabilities under both deterministic and stochastic resource consumption settings. This is achieved by proposing a new complexity term $H_{2, \ell}(Q)$, which simultaneously generalizes the two complexity terms $H_{2, \ell}^{\text{det}}$, $H_{2, \ell}^{\text{sto}}$ on the deterministic and stochastic  consumption settings in \cite{li2024best}. 
Second, we strengthen two lower bound theorems in \cite{li2024best}. Firstly, we generalize the assumption of a lower bound and prove that it holds for any probability distribution on the resource consumption. Secondly, we relax the assumptions in both theorems, requiring a smaller minimum value of available budgets. These additions provide a broader theoretical foundation and demonstrate the versatility and robustness of the proposed approach.

The rest of the paper is organized as follows. In Section \ref{sec:lit-review}, we review topics and analysis correlated to the problem BAIwRC. In Section \ref{sec:model-problem_formulation}, we formulate the BAIwRC model and clarify the adopted assumptions for the analysis. In Section \ref{sec:SH-RR_Algorithm}, we present the proposed algorithm SH-RR with some comments. In Section \ref{sec:main}, we rigorously define the complexity term and prove theoretical guarantee of the algorithm SH-RR. In Section \ref{sec:lower-bound-to-Pr(failure)}, we prove lower bounds for the failure probability, indicating the near optimality of the SH-RR. Section \ref{sec:Numeric-Experiment} includes numeric experiments to validate the excellency of SH-RR.

\section{Literature Review}
\label{sec:lit-review}
The BAI problem has been actively studied in the past decades, prominently under the two settings of fixed confidence and fixed budget. In the fixed confidence setting, the agent aims to minimize the number of arm pulls, while constraining $\Pr(\text{fail BAI})$ to be at most an input confidence parameter. In the fixed budget setting, the agent aims to minimize $\Pr(\text{fail BAI})$, subject to an upper bound on the number of arm pulls. The fixed confidence setting is studied in \citep{EvenMM02,MannorT04,audibert2010best,GabillonGL12,KarninKS13,JamiesonMNB14,KaufmanCG16,GarivierK16}, and surveyed in \citep{JamiesonN14}. The fixed budget setting is studied in \citep{GabillonGL12,KarninKS13,KaufmanCG16,CarpentierL16}. The BAI problem is also studied in the anytime setting \citep{audibert2010best,jun_anytime_2016}, where a BAI strategy is required to recommend an arm after each arm pull. A related objective to BAI is the minimization of simple regret, which is the expected optimality gap of the identified arm, is studied \cite{BubeckMS09, ZhaoSSJ22}. Despite the volume of studies on pure exploration problems on multi-armed bandits, existing works focus on analyzing the total number of arm pulls. We provide a new perspective by considering the \emph{total cost of arm pulls}.

BAI problems with constraints have been studied in various works. \cite{WangWJ22} consider a BAI objective where the identified arm must satisfy a safety constraint. \cite{HouTZ2022} consider a BAI objective where the identified arm must have a variance below a pre-specified threshold. Different from these works that impose constraints on the identified arm, we impose constraints on the exploration process. \cite{SuiGBK15,SuiZBY2018} study BAI problems in the Gaussian bandit setting, with the constraints that each sampled arm must lie in a latent safety set. Our BAI formulation is different in that we impose cumulative resource consumption constraints across all the arm pulls, rather than constraints on each individual pull. In addition, \cite{SuiGBK15,SuiZBY2018}  focus on the comparing against the best arm within a certain reachable arm subset, different from our objective of identifying the best arm out of all arms.

Our work is thematically related to the Bandits with Knapsack problem (BwK), where the agent aims to maximize the total reward instead of identifying the best arm under resource constraints. The BwK problem is proposed in \cite{BadanidiyuruKS18}, and an array of different BwK models have been studied \citep{AgrawalD14,AgrawalD16,SankararamanS18}. In the presence of resource constraints, achieving the optimum under the BwK objective does not lead to BAI. For example, in a BwK instance with single resource constraint, it is optimal to pull an arm with the highest mean reward per unit resource consumption, which is generally not an arm with the highest mean reward, when resource consumption amounts differ across arms. 
\cite{li2023optimal} proposes a BAI problem with BwK setting and shares some settings with us, assuming multiple resources with random consumptions, a finite arm set. But their task is to identify the index set ${\cal X}^*$ of all optimal arms in an LP relaxation to a BwK problem. ${\cal X}^*$ depends on both the mean reward and mean consumption of the arms. The difference of the target marks a significant departure from our methodologies and results in the field.


BAI has also been studied in the Operations Research domain. One of the repeatedly studied topics is Ranking and Selection (R\&S). Similar to BAI, R\&S is commonly divided into two paradigms: the fixed-confidence and fixed-budget settings. In the fixed budget setting, existing R\&S work mainly focus on single budget type and deterministic consumption. For example, \cite{chen2000simulation} design algorithm OCBA to address an R\&S problem, aiming to achieve asymptotically optimal allocation of budget. \cite{glynn2004large} also follow the asymptotically optimal manner. Based on the algorithm OCBA, \cite{cao2025budget} propose a Budget-adaptive algorithm, without requiring the total budget $T$ to be sufficiently large. Still working on the budget allocation, there is also R\&S research aiming at minimizing posterior probability of mistaken identification \cite{qin2025dual}, or pursuing asymptotically optimal procedures without parametric assumptions \cite{feng2022robust}. In R\&S literature, the optimal budget allocation rule is usually correlated to the variance of each arm and optimality is frequently depicted in the asymptotic regime. In contrast, our research adopts the subgaussian assumption, and we conduct analysis with finite budget. In addition, there are also other BAI correlated work in the Operations Research domain. For example, \cite{ahn2025feature} consider the impact of parameter mis-specification under the fixed budget setting, and provide asymptotic analysis. And \cite{ni2017efficient} research on the R\&S within parallel computing scenario.

Lastly, our work is related to the research on cost-aware Bayesian optimization (BO). In this area, an arm might correspond to a hyper-parameter or a combination of hyper-parameters. A widely adopted idea in the BO community is to set up different acquisition functions to guide the selection of sampling points. One of the most popular choices is Expected Improvement (EI) \citep{frazier2018tutorial}, which was designed without considering the heterogeneous resource consumptions. To make EI cost-aware, which corresponds to our setting of a single resource, a popular existing approach is to divide EI by an approximated cost function $c(x)$ \citep{snoek2012practical, poloczek2017multi, swersky2013multi}, leading to the method of Expected Improvement per unit (EIpu). However, \cite{lee2020cost} shows this division may encourage the algorithm to explore domains with low consumption, leading to a worse performance when the optimal point consumes more resources. Then \cite{lee2020cost} designs the Cost Apportioned BO (CArBo) algorithm, whose acquisition function gradually evolves from EIpu to EI. For better performance, \cite{guinet2020pareto} develops Contextual EI to achieve Pareto Efficiency. \cite{abdolshah2019cost} discusses Pareto Front when there are multiple objective functions. \cite{luong2021adaptive} considers EI and EIpu as two arms in a multi-arm bandits problem, using Thompson Sampling to determine which acquisition function is suitable in each round. These works focus on minimizing the total cost, different from our resource constrained setting. In addition, we allow the resource consumption model to be unknown, random and heterogeneous among arms. In comparison, existing works either assume one unit of resource consumed per unit pulled \citep{jamieson2016non,li2020system,bohdal2022pasha,zappella2021resource}, or assume heterogeneity (and deterministic) resource consumption but with the resource consumption (or a good estimate of it) of each arm known \cite{snoek2012practical,ivkin2021cost,lee2020cost}. 

There are also other BO research focusing on the cost-aware setting, with many leveraging the concept of multi-fidelity. \cite{pmlr-v70-kandasamy17a, pmlr-v115-wu20a} assume the targeted objective function $f(x)$ is a slice of function $g(x,s_0)$, i.e. $f(x)=g(x,s_0)$ for some fidelity level $s_0$. The agent is able to evaluate $g(x,s)$ for general $s$ with a lower cost, enabling cheaper approximations of $f$. \cite{foumani2023multi} incorporate categorical variables into the input, by considering there exists a latent manifold. In contrast, \cite{forrester2007multi} assumes the targeted function follows a gaussian process $z_e(x)$, which can be formulated as $z_e(x)=\rho z_c(x)+z_d(x)$. The agent can sample points from Gaussian Process $z_c(x)$ with a cheaper price. All these works aim to exploit multi-fidelity evaluations to extract additional information about the objective function. On the other hand, \cite{belakaria2023bayesian} focuses more on Budget planning, by using Gaussian Process to approximate both the objective and cost functions. Notably, all of the aforementioned studies rely on the correlations across input and some prior knowledge of the cost, which deviates from our main focus.

\textbf{Notation.} For an integer $K>0$, denote $[K] = \{1, \ldots, K\}$. For $d\in [0, 1]$, we denote $\text{Bern}(d)$ as the Bernoulli distribution with mean $d$, and $\mathcal{N}(\mu, \sigma^2)$ as the Gaussian distribution with mean reward $\mu$ and variance $\sigma^2$.

\section{Problem Formulation}
\label{sec:model-problem_formulation}
An instance of Best Arm Identification with Resource Constraints (BAIwRC) is specified by the triple $Q = ([K], C ,\nu= \{\nu_k\}_{k\in [K]})$. The set $[K]$ represents the collection of $K$ arms. There are $L$ types of different resources. The quantity $C = (C_\ell)^L_{\ell=1}\in \mathbb{R}_{>0}^L$ is a vector, and $C_\ell$ is the amount of type $\ell$ resource units available to the agent. For each arm $k\in [K]$, $\nu_k$ is the probability distribution on the $(L+1)$-variate outcome $(R_k; D_{1,k}, \ldots, D_{L, k})$, which is received by the agent when s/he pulls arm $k$ once. By pulling arm $k$ once, the agent earns a random amount $R_k$ of reward, and consumes a random amount $D_{\ell, k}$ of the type-$\ell$ resource, for each $\ell\in\{1, \ldots, L\}$. We allow $R_k, D_{1, k}, \ldots, D_{L, k}$ to be arbitrarily correlated. Denote the mean reward $\mathbb{E}[R_k] = r_k$ for each $k\in [K]$, and denote the mean consumption $\mathbb{E}[D_{\ell, k}] = d_{\ell, k}$ for each $\ell\in [L], k\in [K]$. We assume that $r_k\in[0, 1], \forall k$, and $R_k$ is a 1-sub-Gaussian random variable. In addition, we assume that the marginal probability distribution of 
$\{D_{\ell, k}\}_{k\in [K]}$ belongs to the class of probability distributions  $\mathscr{C}_{b_{\ell}, \sigma^2_{\ell}}=\{\nu:\text{supp}(\nu)\subset [0, 1],\Pr_{D\sim \nu}(|D-\mathbb{E}D|\leq b_{\ell})=1,\text{Var}(\nu)\leq \sigma^2_{\ell}\}$. To make the definition meaningful, we require $\sigma^2_{\ell}\leq b_{\ell}^2$.

 Similar to existing works on BAI, we assume that there is a unique arm with the highest mean reward, and without loss of generality we assume that $r_1 > r_2 \geq \ldots, \geq r_K$. We call arm 1 the optimal arm. We emphasize that the mean consumption amounts $\{d_{\ell, k}\}^K_{k=1}$ on any resource $\ell$ need not be ordered in the same way as the mean rewards. We assume that $d_{\ell, k} > 0$ for all $k\in [K], \ell\in [L]$. Crucially, the quantities $r_k, d_{\ell, k}, \nu_k$ for any $k, \ell$ are not known to the agent. 

\textbf{Dynamics. }The agent pulls arms sequentially in time steps $t = 1, 2, \ldots$, according to a non-anticipatory policy $\pi$. We denote the arm pulled at time $t$ as $A(t)\in [K]$, and the corresponding outcome as $O(t) = (R(t); D_1(t), \ldots, D_L(t))\sim \nu_{A(t)}$. A non-anticipatory policy $\pi$ is represented by the sequence $\{\pi_t\}^{\infty}_{t=1}$, where $\pi_t$ is a function that outputs the arm $A(t)$ by inputting the information collected in time $1, \ldots, t-1$. More precisely, we have $A(t) = \pi_t(H(t-1))$, where $H(t-1) = \{O(s)\}^{t-1}_{s=1}$. The agent stops pulling arms at the end of time step $\tau$, where $\tau$ is a finite stopping time\footnote{For any $t$, the event $\{\tau = t\}$ is $\sigma(H(t))$-measurable, and $\Pr(\tau = \infty) = 0$} with respect to the filtration $\{\sigma(H(t))\}^\infty_{t=1}$. Upon stopping, the agent identifies arm $\psi\in [K] $ to be the best arm, using the information $H(\tau)$. Altogether, the agent's strategy is represented as $(\pi, \tau, \psi)$.

\textbf{Objective.} The agent aims to choose a strategy $(\pi, \tau, \psi)$ to maximize $\Pr(\psi=1)$, the probability of BAI, subject to the resource constraint that $\sum^\tau_{t=1} D_\ell(t) \leq C_\ell$ holds for all $\ell\in [L]$ with certainty. We allow $\{D_{\ell, k}\}_{\ell, k}$ to be arbitrary random variables bounded between 0 and 1, including a special case where $\Pr(D_{\ell, k} = d_{\ell, k})=1$ for all $\ell \in [L], k\in [K]$, which means all the resource consumption amounts are deterministic. In the special case when $L=1$ and $\Pr(D_{1, k}=1 \text{ for all $k\in [K]$}) = 1$, the deterministic consumption setting specializes to the fixed budget BAI problem. 


We focus on bounding the failure probability $\Pr(\text{fail BAI}) = \Pr(\psi\neq 1)$ in terms of the underlying parameters in $Q$. The forthcoming bounds are in the form of $\exp(-\gamma(Q))$, where $\gamma(Q) > 0$ can be understood as a complexity term that encodes the difficulty of the underlying BAIwRC instance $Q$. 
To illustrate, in the case of $L=1$, we aim to bound $\Pr(\psi\neq 1)$ in terms of $\exp(-C_1/ H)$, where $H > 0 $ depends on the latent mean rewards and resource consumption amounts. In the subsequent sections, we establish upper bounds on $\Pr(\psi\neq 1)$ for our proposed strategy SH-RR, as well as lower bounds on $\Pr(\psi\neq 1)$ for any feasible strategy. We demonstrate that the complexity term $\gamma(Q)$ crucially on if $Q$ has deterministic or stochastic consumption.

\section{Algorithm}
\label{sec:SH-RR_Algorithm}
Our proposed algorithm, dubbed Sequential Halving with Resource Rationing (SH-RR), is displayed in  Algorithm \ref{alg:Sequential-Halving}. SH-RR iterates in phases $q \in \{0, \ldots, \lceil \log_2 K\rceil\}$. Phase $q$ starts with a \emph{surviving arm set} $\tilde{S}^{(q)}\subseteq [K]$. After the arm pulling in phase $q$, a subset of arms in $\tilde{S}^{(q)}$ is eliminated, giving rise to  $\tilde{S}^{(q+1)}$. After the final phase, the surviving arm set  $\tilde{S}^{(\lceil \log_2 K\rceil)}$ is a singleton set, and its only constituent arm is recommended as the best arm. 
\begin{algorithm}[tb]
   \caption{Sequential Halving with Resource Rationing (SH-RR)}
   \label{alg:Sequential-Halving}
    \begin{algorithmic}[1]
        \State {\bfseries Input:} Total budget $C$, arm set $[K]$.
        \State {\bfseries Initialize} $\tilde{S}^{(0)} = [K]$, $t=1$.
        \State {\bfseries Initialize} $\textsf{Ration}^{(0)}_\ell = \frac{C_\ell}{\lceil\log_2 K\rceil}$ for each $\ell\in [L]$.
        \For{$q=0$ {\bfseries to} $\lceil\log_2 K\rceil-1$}
            \State {\bfseries Initialize} $I^{(q)}_\ell = 0$ $\forall\ell\in [L]$, $H^{(q)} =J^{(q)} = \emptyset$
            \While{$I^{(q)}_\ell \leq \textsf{Ration}^{(q)}_\ell-1$ for all $\ell\in [L]$}\label{alg:SH-RR-while}
                \State Identify the arm index $a(t)\in \{1, \ldots, |\tilde{S}^{(q)}|\}$ such that $a(t) \equiv t ~\text{mod}~|\tilde{S}^{(q)}|.$\label{alg:rr}
                \State Pull arm $A(t) = k^{(q)}_{a(t)}\in \tilde{S}^{(q)}$. 
                \State Observe the outcome $O(t)\sim \nu_{A(t)}$. 
                \State Update $I^{(q)}_\ell\leftarrow I^{(q)}_\ell + D_\ell(t)$ for each $\ell\in [L]$.
                \State Update $H^{(q)}\leftarrow H^{(q)}\cup \{(A(t),O(t))\}$.
                \State Update $J^{(q)}\leftarrow J^{(q)}\cup\{t\}$.
                \State Update $t\leftarrow t+1$.
           \EndWhile
       \State Use $\cup^q_{m=0} H^{(m)}$ to compute empirical means $\{\hat{r}^{(q)}_k\}_{k\in \tilde{S}^{(q)}}$, see (\ref{eq:emp_mean}).\label{alg:emp_mean}
       \State Set $\tilde{S}^{(q+1)}$ be the set of top $\lceil |\tilde{S}^{(q)}| / 2\rceil$ arms with highest empirical mean.\label{alg:elim}
       \State Set $\textsf{Ration}^{(q+1)}_\ell = \frac{C_\ell}{\lceil\log_2 K\rceil} + ( \textsf{Ration}^{(q)}_\ell-I^{(q)}_\ell)$. \label{alg:ration}
   \EndFor
   \State Output the arm in $\tilde{S}^{(\lceil\log_2 K\rceil)}$.
    \end{algorithmic}
\end{algorithm}
We denote $\tilde{S}^{(q)} = \{k^{(q)}_1, \ldots, k^{(q)}_{|\tilde{S}^{(q)}|}\}$. In each phase $q$, the agent pulls arms in $\tilde{S}^{(q)}$ in a round-robin fashion. At a time step $t$, the agent first identifies (see Line \ref{alg:rr}) the arm index $a(t)\in \{1, \ldots, |\tilde{S}^{(q)}|\}$ and pulls the arm $k^{(q)}_{a(t)}\in \tilde{S}^{(q)}$. The round robin schedule ensures that the arms in $\tilde{S}^{(q)}$ are uniformly explored. SH-RR keeps track of the amount of type-$\ell$ resource consumption via $I^{(q)}_\ell$. The \textbf{while} condition (see Line \ref{alg:SH-RR-while}) ensures that at the end of phase $q$, the total amount $I^{(q)}_\ell$ of type-$\ell$ resource consumption during phase $q$ lies in $ (\textsf{Ration}^{(q)}_{\ell}-1, \textsf{Ration}^{(q)}_{\ell}]$ for each $\ell\in [L]$. The lower bound ensures sufficient exploration on $\tilde{S}^{(q)}$, while the upper bound ensures the feasibility of SH-RR to the resource constraints, as formalized in the following claim:
\begin{claim}\label{claim:feasibility}
With certainty, SH-RR consumes at most $C_\ell$ units of resource $\ell$, for each $\ell\in [L]$. 
\end{claim}

Proof of Claim \ref{claim:feasibility} is in Appendix \ref{app:claim_feasibility}. 
Crucially, SH-RR maintains the observation history $H^{(q)}$ that is used to determined the arms to be eliminated from $\tilde{S}^{(q)}$. 
After exiting the \textbf{while} loop, the agent computes (in Line \ref{alg:emp_mean}) the empirical mean
\begin{align}\label{eq:emp_mean}
\hat{r}^{(q)}_k = \frac{\sum^q_{m=0}\sum_{t\in J^{(m)}} R(t) \cdot \mathbbm{1}(A(t) = k)}{\max\{\sum^q_{m=0}\sum_{t\in J^{(m)}} \mathbbm{1}(A(t) = k), 1\}} 
\end{align}
for each $k\in \tilde{S}^{(q)}$. The surviving arm set $\tilde{S}^{(q+1)}$ in the next phase of phase $q+1$ consists of the $\lceil |\tilde{S}^{(q)}| / 2\rceil$ arms in $\tilde{S}^{(q)}$ with the highest empirical means, see Line \ref{alg:elim}. The amounts of resources rationed for phase $q+1$ is in Line \ref{alg:ration}.

\section{Performance Guarantees of SH-RR}
\label{sec:main}
We start with some necessary notation. For each $k\in \{2, \ldots, K\}$, we denote $\Delta_k = r_1 - r_k\in [0, 1]$. We also denote $\Delta_1 = r_1 - r_2 = \Delta_2$. Consequently, we have $\Delta_1 =  \Delta_2 \leq \Delta_3\leq \ldots \leq \Delta_K$. 
For each resource type $\ell\in [L]$, we denote $d_{\ell, (1)},d_{\ell, (2)}, \ldots, d_{\ell, (K)}$ as a permutation of $d_{\ell, 1}, d_{\ell, 2},\ldots, d_{\ell, K}$ such that $d_{\ell, (1)}\geq  d_{\ell, (2)} \geq \ldots \geq d_{\ell, (K)}$. Based on the above notation, we define
\begin{align}
    \label{eq:def_f}
    f(b,\sigma, d)=\frac{4b}{\log (\frac{4b^2}{\sigma^2}+1)}+d,
\end{align}
and
\begin{align}
    H_{2, \ell}(Q)= & \max_{2\leq k\leq K}\frac{\sum_{a=1}^k f(b_{\ell},\sigma_{\ell},  d_{\ell,(a)})}{\Delta_k^2}, \label{eq:sto_H}\\
    H_{2, \ell}^{\text{det}}(Q)= & \max_{2\leq k\leq K}\frac{\sum_{a=1}^k d_{\ell,(a)}}{\Delta_k^2}, \label{eq:det_H}
\end{align}
Definition (\ref{eq:sto_H}) works for any 
$\{D_{\ell, k}\}_{k\in [K]}$ 
whose distribution is in class $\mathscr{C}_{b_{\ell}, \sigma^2_{\ell}}=\{\nu:\text{supp}(\nu)\subset [0, 1],\Pr_{D\sim \nu}(|D-\mathbb{E}D|\leq b_{\ell})=1,\text{Var}(\nu)\leq \sigma^2_{\ell}\}$. In the case that the consumptions are deterministic, i.e $\sigma_{\ell}=0$ holds for all $\ell\in[L]$, we have $f(b,0,d)=d$ and $H_{2, \ell}(Q)= H_{2, \ell}^{\text{det}}(Q)$. The formulation is consistent with the fixed budget BAI setting \cite{audibert2010best,KarninKS13}. 

Our first main result is an upper bound on $\Pr(\text{fail BAI}) = \Pr(\psi\neq 1)$ for our proposed SH-RR.
\begin{theorem}
    \label{theorem:upper-bound-of-failure}
    Consider a BAIwRC instance $Q$. The SH-RR algorithm has BAI failure probability $\Pr(\psi \neq 1)$ at most
    \begin{align}\label{eq:upper-bound-of-failure}
    2LK (\log_2 K) \exp\left(-\frac{1}{4\lceil \log_2 K\rceil}\cdot\gamma^{}(Q) \right),
    \end{align}
    where  $\gamma(Q) = \min_{\ell\in [L]}\{C_\ell / H_{2, \ell}(Q)\}$, and $H_{2, \ell}(Q)$ is defined in (\ref{eq:sto_H}).
\end{theorem}
In the case that the mean consumptions are deterministic, we can further remove the factor $L$ outside the exponential function. Please check appendix \ref{appendix:theorem_upper} for the Theorem \ref{theorem:upper-bound-of-failure-det} and details. We present a sketch proof of Theorem \ref{theorem:upper-bound-of-failure} at the end of this section.

The intuition behind Theorem \ref{theorem:upper-bound-of-failure} consists of 2 observations. The first observation is its connection with the result in \cite{KarninKS13}, which takes $L=1$, $d_{1,1}=d_{1,2}=\cdots=d_{1,K}$ and assume the deterministic consumption. Consistent with the above discussion,  our complexity term simplifies, i.e. $H_{2,1}(Q)=H_{2,1}^{\text{det}}(Q)$, and the upper bound in Equation \ref{eq:upper-bound-of-failure} closely resembles Theorem 4.1 in \cite{KarninKS13}. The only remaining differences are the constant inside the exponential term and the presence of a factor $K$ outside the exponential. These variations arise because, unlike \cite{KarninKS13}, our analysis retains the history of arm pulls across phases, which leads to a slightly different accounting of sample complexity. 

The second observation is performance guarantee of SH-RR improves when the complexity term $\gamma(Q)$ increases. We provide intuitions in the special case of $L = 1$, so $\ell=1$ always. In this case, $\gamma(Q)=\frac{C_1}{H_{2,1}}$. The upper bound (\ref{eq:upper-bound-of-failure}) decreases when $C_1$ increases, since more resource units allows more experimentation, hence a lower failure probability. The upper bound (\ref{eq:upper-bound-of-failure}) increases when $H_{2, 1}(Q)$ increases. Indeed, when $\Delta_k$ decreases, more arm pulls are needed to distinguish between arms $1,k$, leading to a higher $\Pr(\text{fail BAI})$. In addition, $H_{2,1}(Q)$ and $H_{2,1}^{\text{det}}(Q)$ both grow linearly with $\{d_{1, (a)}\}_{a=1}^K$. When $\{d_{1, (a)}\}_{a=1}^K$ increase, the agent consumes more resource units when pulling the arm with the $k$-th highest consumption, which leads to less arm pulls under a fixed budget. Beyond the linear part with $d_{1,(a)}$, the upper bound in (\ref{eq:upper-bound-of-failure}) is also correlated to the $b_{1}, \sigma_1$. Crucially, the expected consumption $d_{1, (j)}$ in $H^{\text{det}}_{2, 1}(Q)$ is replaced with \emph{effective consumption} $f(b_{1}, \sigma_{1}, d_{\ell, (j)})$ in $H_{2, 1}(Q)$. The non-linearity of $f$ regarding $b_{1}, \sigma_{1}$ encapsulates the impact of randomness in resource consumption. The function $f$ is increasing in $\sigma_{1}$. If the variance of consumption $\sigma^2_1$ decreases, $f(b_1, \sigma_1, d_{1,(a)})$ also decreases, which corresponds to a higher probability of more pulling times of an arm, under a fixed budget.  

To illustrate the comparison between $f(b, \sigma, d)$ and $d$, we provide three more examples. \textbf{(1)} If $D\sim \text{Unif(0, 2d)}$, then $\mathbb{E}D=d$. Here $b=d$ and $\sigma^2=\text{Var}(D)=\frac{d^2}{3}$, we get $f(b,\sigma, d)=\frac{4d}{\log (\frac{4d^2}{d^2/3}+1)}+d=\Theta(d)$. \textbf{(2)} If $\Pr(D=2d)=\frac{1}{2}, \Pr(D=0)=\frac{1}{2}$, then $\mathbb{E}D=d$. Here $b=d$ and $\sigma^2=\text{Var}(D)=\frac{d^2}{2}$, we get $f(b,\sigma, d)=\frac{4d}{\log (\frac{4d^2}{d^2/2}+1)}+d=\Theta(d)$. More generally,  $b=O(d)$ implies $f(b, \sigma, d)=\Theta(d)$, by the assumption $\frac{b^2}{\sigma^2}\geq 1$. But things will change if $b=O(d)$ doesn't hold. See the third example. \textbf{(3)} If $D\sim \text{Bern}(d)$, then $\mathbb{E}D=d$, we can take $b=1$ and $\sigma^2=\text{Var}(D)=d(1-d)$, we get $f(b,\sigma, d)=\frac{4}{\log (\frac{4}{d(1-d)}+1)}+d=O(\frac{4}{\log (\frac{4}{d}+1)})$. In this case, $\lim_{d\rightarrow 0}\frac{f(b,\sigma, d)}{d}=+\infty$, suggesting that $H_{2, 1}(Q)$ can be strictly larger than $H_{2, 1}^{\text{det}}(Q)$ if $d$ is sufficiently small with Bernoulli resource consumption. 

We conduct a simulation on two instances $Q^{\text{det}}, Q^{\text{sto}}$ to further present the fundamental difference between stochastic and deterministic setting. Instances $Q^{\text{det}}, Q^{\text{sto}}$ both have with $K=2, L = 1, C=2$. Instances $Q^{\text{det}}, Q^{\text{sto}}$ share the same Bernoulli rewards with means $r_1 = 0.5, r_2 = 0.4$ and the same mean resource consumption $d_1 = d_2 = d$, where $d$ varies. In $Q^{\text{det}}$, an arm pull consumes $d$ units with certainty, while in $Q^{\text{sto}}$ it consumes Bern($d$) per pull. We plot $\log(\Pr(\psi \neq 1))$ under SH-RR against the varying $d$ in Figure \ref{fig:compare}, while other model parameters are fixed.
\begin{figure}
    \centering
    \begin{subfigure}
      \centering
      \includegraphics[width=.48\linewidth]{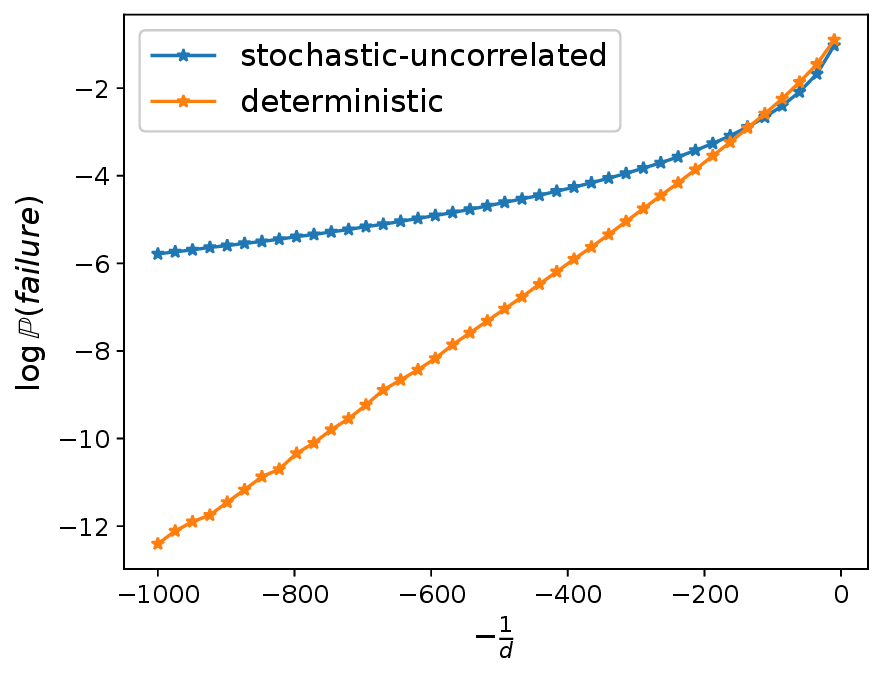}
    \end{subfigure}%
    \begin{subfigure}
      \centering
      \includegraphics[width=.48\linewidth]{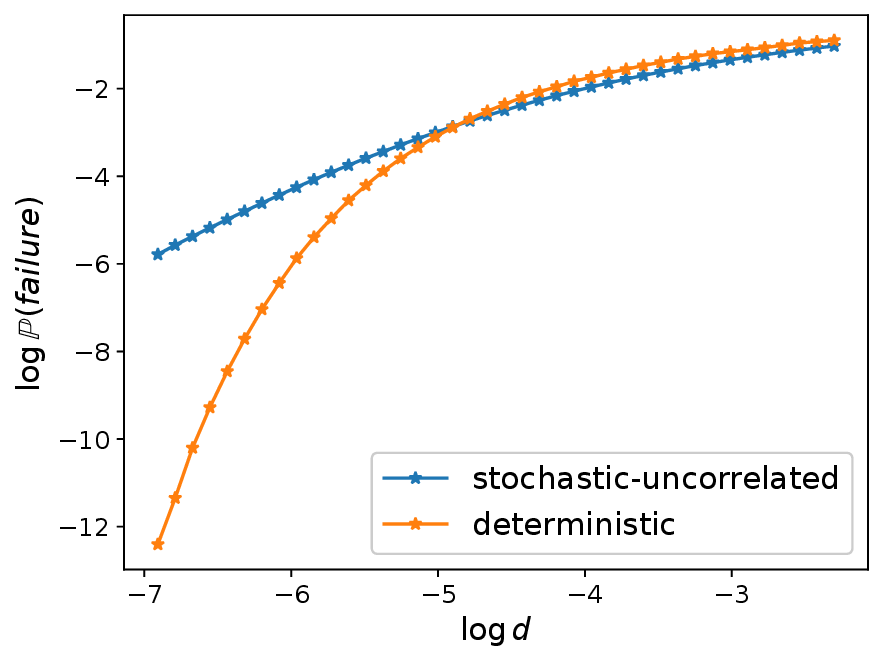}
    \end{subfigure}
    \caption{Convergence rates of $\log(\Pr(\psi \neq 1))$, with $10^7$ repeated trials}
    \label{fig:compare}
\end{figure}
Figure \ref{fig:compare} shows that the $\Pr(\psi \neq 1)$ for $Q^{\text{det}}$ is always less than that for $Q^{\text{sto}}$. In addition, $\Pr(\psi \neq 1)$'s for $Q^{\text{det}}, Q^{\text{sto}}$ diverge when $d$ shrinks,, suggesting that the problem nature is different. The left panel shows that the plotted $\log(\Pr(\psi \neq 1))$ does not decreases linearly as $1/d$ grows, which implies that the bound in (\ref{eq:upper-bound-of-failure-det}) does not hold for $Q^{\text{sto}}$ when $d$ is sufficiently small.

Besides the difference of $H$ notation, the second observation is the difference of the factor outside the exponential function. There is an extra factor $L$ in Theorem \ref{theorem:upper-bound-of-failure}, compared to the Theorem \ref{theorem:upper-bound-of-failure-det}. The intuition is no matter how large $L$ is, the bottle neck resource type that determines the total pulling times is fixed, in the case of deterministic consumption setting. For example, if $C_1=\cdots=C_L$ and $d_{1,(a)}=\cdots=d_{L,(a)}$ holds for all $a\in[K]$, there is no difference regarding the value of $L$, as the resource type $\ell=1$ can be considered the bottle neck to restrict the pulling times of each arm. But the same example makes a difference for stochastic consumption setting. As $L$ increases, the maximum consumption across different resource types when pulling an arm also increases with higher probability. In conclusion, a larger $L$ increases the probability of smaller pulling times, making the failure probability larger.

Observe that $H^{\text{det}}_{2,\ell}(Q), H^{}_{2, \ell}(Q)$ involves the mean consumption $d_{\ell, (1)}, \ldots, d_{\ell, (K)}$ in a non-increasing order, providing a worst-case hardness measure over all permutations of the arms. One could wonder if the definition of $H_{2,\ell}^{\text{det}}, H^{}_{2, \ell}$ can be refined in the non-ordered way, i.e. $\max_{k\in \{2, \ldots, K\}}\{ \sum^k_{j=1} d_{\ell, j} /\Delta_k^2 \}$ or $\max_{k\in \{2, \ldots, K\}}\{ \sum^k_{j=1} f(b_{\ell},\sigma_{\ell}, d_{\ell, j}) /\Delta_k^2 \}$. Our analysis in appendix \ref{sec:Improvement_of_H2_is_unachievable} shows such a refinement is unachievable.

At the end of this section, we provide a sketch proof of Theorem \ref{theorem:upper-bound-of-failure}.
\begin{proof}{Sketch Proof of Theorem \ref{theorem:upper-bound-of-failure}}
    Denote $T^{(q)}$ as the pulling times of each survival arm at phase $q$, and $\bar{T}^{(q)} = T^{(1)} + \ldots + T^{(q)}$ as the total pulling times of each survival arm at the end of phase $q$. Define $E_q:=\left\{i: \hat{r}_{i,\bar{T}^{(q)}}> \hat{r}_{1, \bar{T}^{(q)} }\right\}$ and the bad event $B^{(q)}=\left\{ |E_q| \ge \lceil\frac{K}{2^{q+1}}\rceil\right\}$.

    \textbf{Step 1.} We prove an upper bound for the $\Pr(B^{(q)})$ by splitting it into two parts.
    \begin{align*}
        &\Pr(B^{(q)})\\
        \leq &\underbrace{ \Pr\left(\exists k \ge \lceil\frac{K}{2^{q+1}}\rceil, \hat{r}_{k, \bar{T}^{(q)}} > \hat{r}_{1, \bar{T}^{(q)}}, \bar{T}^{(q)} \ge \bar{\beta}^{(q)} \right) }_{(\P)}\\
        &\qquad \qquad \qquad \qquad+ \underbrace{ \Pr\left(\bar{T}^{(q)} < \bar{\beta}^{(q)}\right) }_{(\ddagger)}
    \end{align*}
    where $\bar{\beta}^{(q)} = \min_{\ell\in [L]}\{\beta^{(q)}_\ell\}$ and $\beta^{(q)}_\ell:=\frac{C_{\ell}}{ \lceil \log_2 K\rceil\left(\sum_{k=1}^{\lceil \frac{K}{2^q} \rceil} f(b_{\ell}, \sigma_{\ell}, d_{\ell, (k)}) \right)}$. Intuitively speaking, $(\P)$ denotes probability that arm 1 might get eliminated from the survival set $\tilde{S}^{(q)}$, when the pulling times of arm 1 is large enough. And $(\ddagger)$ denotes the probability that pulling time is not large enough.

    \textbf{Step 2.} By applying concentration theorems, we show that
    \begin{align*}
        \Pr\left(\bar{T}^{(q)} < \bar{\beta}^{(q)}\right) \leq & L\exp\left(-\bar{\beta}^{(q)}\right)  \\
        (\P) \leq & K \exp\left(-\bar{\beta}^{(q)} \frac{(r_1-r_{ \lceil\frac{K}{2^{q+1}}\rceil })^2}{2}\right).
    \end{align*}
    
    \textbf{Step 3.} Merging step 1 and 2, we can conclude
    \begin{align*}
        &\Pr(\hat{k}\neq 1)\\
        \le &\Pr(\cup_{q=1}^{\log_2 K} B^{(q)})\\
        \le &\sum_{q=1}^{\log_2 K} \Pr(B^{(q)})\\
        \le &2LK (\log_2 K) \exp\left(-\frac{1}{4}\min_{\ell \in[L]} \{ \frac{C_{\ell}}{\lceil\log_2 K\rceil H_{2, \ell}^{}}\}\right).
    \end{align*}
\end{proof}

\section{Lower Bound of $\Pr(\psi \neq 1)$ }
\label{sec:lower-bound-to-Pr(failure)}

We complement our analysis with matching lower bounds for the BAIwRC problem, thereby establishing the near-optimality of SH-RR. Unlike the SH-RR analysis, our lower bounds allow arm 1 to be suboptimal. In what follows, we consider instances where arm 1 needs not be optimal, different from our development of SH-RR. Our lower bound results involve constructing $K$ instances $\{Q^{(i)}\}^K_{i=1}$, where the resource consumption models are carefully crafted to (nearly) match the upper bound results of SH-RR. In the next two subsections, we present two lower bounds for different consumption settings. The first lower bound applies for any consumption distribution, suggesting that the definition of equation \ref{eq:det_H} is appropriate, and confirming Theorem \ref{theorem:upper-bound-of-failure} are tight. The second lower bound, specific to Bernoulli distribution, indicates the necessity of the term $\frac{4b}{\log (\frac{4b^2}{\sigma^2}+1)}$ in the equation \ref{eq:sto_H}, further implying the tightness of Theorem \ref{theorem:upper-bound-of-failure}. Taken together, these results confirm that SH-RR achieves near-optimal performance across a broad range of settings.

\subsection{Lower Bound for General Consumption Setting}
For any $K\geq 2, K\in\mathbb{N}$, we firstly create $K$ instances. Each instance involves the set $[K]$ of $K$ arms and $L$ types of resources $[L]$.  Let $\{r_k\}^K_{k=1}$ be any sequence such that (\hypertarget{prop:a}{a}) $1/2 = r_1\ge r_2 \ge \cdots \ge r_K >0$, and let $\{ \{\nu^{\text{cons}}_{\ell, (k)}\}^{K}_{k=1} \}_{\ell\in [L]}$ be a fixed but arbitrary distribution collection of $L$ sequences such that their mean values $d_{\ell, (k)}=\mathbb{E}_{D\sim \nu^{\text{cons}}_{\ell, (k)}} D$ satisfy (\hypertarget{prop:b}{b}) $d_{\ell,(1)}\ge d_{\ell,(2)}\ge \cdots \ge d_{\ell,(K)}$ for all $\ell\in [L]$, and $d_{\ell, (k)}\in (0, 1]$ for all $\ell\in [L], k\in [K]$. 

In instance $Q^{(i)}$, pulling arm $k\in [K]$ generates a random reward $R_k\sim \mathcal{N}(r^{(i)}_k, 1)$, where
$$r^{(i)}_k = \bigg\{ \begin{array}{ll} r_k  & \;\text{if } k\ne i,\\
1 - r_k & \;\text{if } k= i,\end{array}.$$
\textcolor{blue}{Pulling arm $k\in [K]$
consumes a random variable $D$ units of resource $\ell$, following
\begin{equation}\label{eq:consumption_model}
D \sim \Bigg\{\begin{array}{ll}\nu^{\text{cons}}_{\ell,(2)}&\;\text{if }k=1,\\ \nu^{\text{cons}}_{\ell,(1)}&\;\text{if }k=2,\\\nu^{\text{cons}}_{\ell,(k)}&\;\text{if }k\in \{3, \ldots, K\}\end{array}
\end{equation}
for each $\ell\in [L]$.} 

In instance $Q^{(i)}$, arm $i$ is the uniquely optimal arm. All instances $Q^{(1)}, \ldots, Q^{(K)}$ have identical resource consumption model, since the consumption amounts (\ref{eq:consumption_model}) do not depend on the instance index $i$. This ensures that no strategy can extract information about the reward from an arm's consumption. In addition, the consumption amounts in (\ref{eq:consumption_model}) are designed to ensure that (a) instance $Q^{(1)}$ is the hardest among $\{Q^{(i)}\}^K_{i=K}$ in the sense that $H^\text{det}_{2, \ell}(Q^{(1)}) = \max_{i\in [K]}H^\text{det}_{2, \ell}(Q^{(i)})$ for every $\ell\in [L]$, (b) The ordering $d_{\ell,1} \leq d_{\ell,2} \geq d_{\ell,3}\geq \ldots, \geq d_{\ell,K}$ makes $Q^{(1)}$ a hard instance in the sense that it cost the most to distinguish the second best arm (arm 2) from the best arm. More generally, for each resource $\ell$, the consumption amounts are designed such that a sub-optimal arm is more costly to pull when its mean reward is closer to the optimum. Our construction leads to the following lower bound on the performance of any strategy:
\begin{theorem}
    \label{theorem:lower-bound-fixed-consumption-multiple-resource}
    Consider general consumption instances $Q^{(1)}, \ldots, Q^{(K)}$ constructed as above, with $\{r_k\}_{k\in [K]}, \{d_{\ell, (k)}\}_{\ell, k}$ being fixed but arbitrary sequences of parameters that satisfy properties (\hyperlink{prop:a}{a}, \hyperlink{prop:b}{b}) respectively. 
    When $C_1, \ldots, C_L$ are sufficiently large, such that $96\cdot\text{min}_{i\in [K],\ell\in [L]} \left\{\frac{2 C_\ell}{H^\text{det}_{2, \ell}(Q^{(i)})}\right\} \geq 1$, for any strategy there exists an instance $Q^{(i)}$ (where $i\in [K]$) such that 
    \begin{align*}
        \Pr_{i}(\psi\ne i)\ge &\frac{1}{6}\exp\left(-108\cdot \gamma^{\text{det}}(Q^{(i)})\right),
    \end{align*}
    where $\Pr_{i}(\cdot)$ is the probability measure over the trajectory $\{(A(t), O(t))\}^\tau_{t=1}$ under which the arms are chosen according to the strategy and the outcomes are modeled by $Q^{(i)}$, and $\gamma^{\text{det}}(Q)$ is as defined in Theorem \ref{theorem:upper-bound-of-failure-det}.
\end{theorem}
Detailed proof is in Appendix \ref{pf:thm_low_det}, and we present the sketch proof at the end of this subsection. Compared to the Theorem 4 in \cite{anonymous2024best}, Theorem \ref{theorem:lower-bound-fixed-consumption-multiple-resource} improve the result from two perspectives. 

First, Theorem \ref{theorem:lower-bound-fixed-consumption-multiple-resource} works for \text{all consumption distribution}, while the Theorem 4 in \cite{anonymous2024best} only proves the correctness for the deterministic consumption setting. Together with the Theorem \ref{theorem:upper-bound-of-failure-det}, Theorem \ref{theorem:lower-bound-fixed-consumption-multiple-resource} demonstrates the \textbf{near-optimality of SH-RR}, and the fundamental importance of the quantity $\gamma^{\text{det}}(Q)$ for the BAIwRC problem. More precisely, the bounds in Theorems \ref{theorem:upper-bound-of-failure-det}, \ref{theorem:lower-bound-fixed-consumption-multiple-resource} imply
\begin{align}
&\underset{\text{strategy}}{\sup} ~\underset{\substack{\text{det inst $Q$:}\\ \gamma^{\text{det}}(Q) \geq \kappa^{\text{det}}}}{\inf}\left\{\frac{- \log\left( \Pr(\text{fail BAI})\right) }{\gamma^{\text{det}}(Q)} \right\} \in\nonumber\\
&\left[\frac{1}{4 \lceil\log_2 K\rceil}, 108\right]\label{eq:sup_inf_det}, 
\end{align}
where $\kappa^{\text{det}} = K$. What's more, from the discussion in section \ref{sec:main}, we know $\gamma(Q) =O(\gamma^{\text{det}}(Q))$ holds for some resource consumption contribution. For example, $\nu^{\text{cons}}_{\ell,(i)} = \text{Unifom}(0, 2d_{\ell, (k)})\Rightarrow \gamma(Q)=(\frac{4}{\log 13}+1)\gamma^{\text{det}}(Q))$, which implies Theorem \ref{theorem:upper-bound-of-failure-det} is nearly optimal for distribution $\nu^{\text{cons}}_{\ell,(i)} = \text{Unifom}(0, 2d_{\ell, (k)})$. For another example, $\Pr(D=2d_{\ell, (k)})=\frac{1}{2}, \Pr(D=0)=\frac{1}{2}\Rightarrow (\frac{4}{\log 8}+1)\gamma^{\text{det}}(Q))$, which also suggests the near optimality of Thoerem \ref{theorem:upper-bound-of-failure-det}. Combining Theorem \ref{theorem:upper-bound-of-failure-det} and \ref{theorem:lower-bound-fixed-consumption-multiple-resource}, we know the algorithm SH-RR is also nearly optimal.

Second, Theorem \ref{theorem:lower-bound-fixed-consumption-multiple-resource} requires a smaller minimum value of budget $\{C_{\ell}\}_{\ell=1}^L$. Theorem 4 in \cite{anonymous2024best} requires $\{C_{\ell}\}_{\ell=1}^L$ are large enough, such that $\text{min}_{i\in [K],\ell\in [L]} \left\{\frac{2 C_\ell}{H^\text{det}_{\ell, 2}(Q_i)}\right\} \geq \sqrt{\min_{\ell\in [L]}\left\{\lfloor\frac{C_{\ell}}{d_{{\ell},(K)}}\rfloor\right\}\log\left(12K \min_{\ell\in [L]}\left\{\lfloor\frac{C_{\ell}}{d_{{\ell},(K)}}\rfloor\right\}\right)}$. While in Theorem \ref{theorem:lower-bound-fixed-consumption-multiple-resource}, it only requires $96\cdot\text{min}_{i\in [K],\ell\in [L]} \left\{\frac{2 C_\ell}{H^\text{det}_{2, \ell}(Q^{(i)})}\right\} \geq 1$. This improvement is achieved by a better usage of transportation equality. Details are in Appendix \ref{pf:thm_low_det}.

Following is the sketch proof.
\begin{proof}{Sketch Proof of Theorem \ref{theorem:lower-bound-fixed-consumption-multiple-resource}}
     The overall proof consists of two steps.
    
    \textbf{Step 1.} We show that, under the assumption $\Pr_1(\psi\neq 1) <1/2$, for every $i\in \{2, \ldots, K\}$ it holds that
    \begin{align*}
        & \Pr_i(\psi\neq i) \geq \\
       & \frac{1}{6}\exp\left(-12t_i(\frac{1}{2}-r_i)^2-\sqrt{96t_i(\frac{1}{2}-r_i)^2}\right),
    \end{align*}
    where 
    \begin{align*}
        t_i = \mathbb{E}_1[T_i], \quad T_i = \sum^\tau_{t=1} \mathbf{1}(A(t) = i)
    \end{align*}
    is the number of times pulling arm $i$. Note that if the assumption $\Pr_1(\psi\neq 1) <1/2$ is violated, the conclusion in Theorem \ref{theorem:lower-bound-fixed-consumption-multiple-resource} immediately holds for $Q^{(1)}$ and sufficiently large enough $\{C_{\ell}\}_{\ell=1}^L$. The main idea is to establish step 1 is applying the transportation equality in \cite{CarpentierL16} 
    \begin{align*}
        & \Pr_i(\psi \neq i)\\
        = & \mathbb{E}_{1}\left(\mathbbm{1}(\psi \neq i)\exp\left(-\sum_{s=1}^{T_i}\log\frac{d\nu_i^{(1)}}{d\nu_i^{(i)}}(\tilde{R}_i(s))\right)\right),
    \end{align*}
    where $\tilde{R}_i(s)\sim N(r_i, 1)$, $\nu_i^{(1)}=N(r_i, 1)$, $\nu_i^{(i)}=N(1-r_i, 1)$. Besides this equality, we need to further prove probability bound for the concentration event of realized KL-divergence. Details are omitted here.

    \textbf{Step 2.} We show that there exists $i\in \{2, \ldots K\}$ such that 
    \begin{align*}
        t_i (1/2 - r_i)^2 & \leq \text{min}_{\ell\in [L]} \left\{\frac{2 C_\ell}{H^\text{det}_{2, \ell}(Q^{(1)})}\right\}\\
        &  \leq \text{min}_{\ell\in [L]} \left\{\frac{2 C_\ell}{H^\text{det}_{2, \ell}(Q^{(i)})}\right\}.
    \end{align*}
    To prove this, we apply Wald's Equality to show 
    \begin{align*}
        t_1d_{\ell,(2)}+ t_2d_{\ell,(1)} + \sum_{k=3}^K t_k d_{\ell,(k)}\le C_{\ell}
    \end{align*}
    holds for all $\ell\in [L]$. Further, we can conclude for all $\ell\in [L]$,
    \begin{align*}
        &\frac{2 C_{\ell} d_{\ell,(1)}}{H^\text{det}_{1, \ell}(Q^{(1)})(\frac{1}{2}-r_2)^2} + \sum_{k=3}^K \frac{C_{\ell} d_{\ell,(k)}}{H^\text{det}_{1, \ell}(Q^{(1)})(\frac{1}{2}-r_k)^2}\nonumber\\
        \ge &t_1d_{\ell,(2)}+t_2d_{\ell,(1)}+t_3d_{\ell,(3)}+\cdots+t_Kd_{\ell,(K)},
    \end{align*}
    which implies $\exists i\in \{2, \ldots K\}$, such that $t_i (1/2 - r_i)^2\leq \text{min}_{\ell\in [L]} \left\{\frac{2 C_\ell}{H^\text{det}_{2, \ell}(Q^{(i)})}\right\}$.

    Merging the step 1 and 2, we prove Theorem \ref{theorem:lower-bound-fixed-consumption-multiple-resource}.
\end{proof}

\subsection{Lower Bound for Bernoulli Consumption Setting}
Theorem \ref{theorem:lower-bound-fixed-consumption-multiple-resource} works for any consumption distribution, suggesting the definition of equation (\ref{eq:det_H}) is appropriate. In addition, if the consumption distributions are all Bernoulli and the mean consumption value are sufficiently small, we are able to prove a stronger lower bound.
\begin{theorem}
    \label{theorem:lower-bound-sto-consumption-multiple-resource}
    Consider a fixed but arbitrary function $g:[0, +\infty) \rightarrow [0, +\infty)$ that is increasing and $\lim\limits_{d\rightarrow 0^+}\frac{1}{g(d)\log\frac{1}{d}}=+\infty$, $g(0)=0$, as well as any fixed $\{r_k\}_{k=1}^K \subset (0, 1)$, $\frac{1}{2}=r_1 > r_2\geq \cdots \geq r_K=\frac{1}{4}$, any fixed $\{d^0_{\ell, (k)}\}_{k=1, \ell=1}^{K, L} \subset \mathbb{R}$, $d^{0}_{\ell, (1)} \geq d^{0}_{\ell, (2)} \geq \cdots\geq d^{0}_{\ell, (K)}$, and any fixed $i\in\{2,\cdots, K\}$. We can identify $\bar{c} \in (0, 1)$, such that for any $c\in (0, \bar{c})$  and large enough $\{C_{\ell}\}_{\ell=1}^L$, such that $C_{\ell}\geq 64$ holds for all $\ell\in [L]$. By taking $d_{\ell, (j)}= cd^{0}_{\ell, (j)}, \forall j\in [K], \forall \ell\in[L]$, we can construct corresponding instances $Q^{(j)}$: (1) pulling arm $k\in [K]$ generates a random reward $R_k\sim \mathcal{N}(r^{(j)}_k, 1)$, where $r^{(j)}_k = \bigg\{ \begin{array}{ll} r_k  & \;\text{if } k\ne j,\\1 - r_k & \;\text{if } k= j,\end{array}$, (2) pulling arm $k\in [K]$ consumes $D_{\ell}\sim\text{Bern}(d_{\ell, k})$,
    $d_{\ell, k} =d_{\ell,(k)}$ 
    for $\ell\in [L]$. The following performance lower bound holds for any strategy:
    \begin{align*}
        \max_{j\in \{1, i\}}\Pr_{Q^{(j)}}(\psi\neq j) \geq \exp\left(-\tilde{\gamma}^{\text{sto}}(Q^{(j)})\right),
    \end{align*}
    where $\tilde{\gamma}^{\text{sto}}(Q^{(j)}) = \min_{\ell\in [L]}\frac{C_{\ell}}{\tilde{H}_{2, \ell}^{\text{sto}}}$, and $\tilde{H}_{2, \ell}^{\text{sto}}=\max\limits_{k\in\{2,3,\cdots, K\}} \frac{\sum_{j=1}^k g(d_{\ell, (j)})}{\Delta_k^2}$.
\end{theorem}
Detailed proof is in Appendix \ref{pf:thm_low_sto}. Similar to the last subsection, we make some comments on Theorem \ref{theorem:lower-bound-sto-consumption-multiple-resource} before elaborating its sketch proof. In Theorem \ref{theorem:lower-bound-sto-consumption-multiple-resource}, we establish a lower bound for Bernoulli consumption in multiple resource scenarios. Since this is a special case of stochastic consumption, we can also claim the near-optimality of the SH-RR, alongside Theorem \ref{theorem:upper-bound-of-failure}.

In Theorem \ref{theorem:upper-bound-of-failure}, we introduce a novel complexity measure, $H_{2, \ell}(Q)$. This measure is notable for incorporating a term $\frac{4b}{\log (\frac{4b^2}{\sigma^2}+1)}$. As mentioned in the Section \ref{sec:main}, if the resource consumption follows Bernoulli Distribution with mean value $d$, $\frac{4b}{\log (\frac{4b^2}{\sigma^2}+1)}=\Theta\left(\frac{1}{\log \frac{1}{d}}\right)$, when $d$ is small. This results in a larger and possibly weaker upper bound compared to the deterministic case and is also different from traditional BAI literature. A pertinent question arises: Is it feasible to refine this term from $\frac{1}{\log\frac{1}{d}}$ to $d$?

Theorem \ref{theorem:lower-bound-sto-consumption-multiple-resource} directly addresses this question, clarifying that such a refinement should not be expected to hold for any given sets $\{r_k\}_{k=1}^K$ and $\{d_{\ell, k}\}_{k=1, \ell=1}^{K, L}$. This result decisively indicates that the term $\frac{4b}{\log (\frac{4b^2}{\sigma^2}+1)}$ in the definition of $H_{2, \ell}(Q)$ is irreplaceable with $\left(\frac{4b}{\log (\frac{4b^2}{\sigma^2}+1)}\right)^{1+\varepsilon} $ for any $\varepsilon > 0$, and certainly not with $d$, which also suggests that the lower bound in Theorem \ref{theorem:lower-bound-fixed-consumption-multiple-resource} is weaker. 

The intuition of the difference between the Theorem \ref{theorem:lower-bound-fixed-consumption-multiple-resource} and \ref{theorem:lower-bound-sto-consumption-multiple-resource} is as follows. Theorem \ref{theorem:lower-bound-fixed-consumption-multiple-resource} concerns about the allocation of resources. There always exists an arm whose allocated resource is relatively insufficient, resulting a lower bound of failure probability when this arm is indeed the optimal arm. While Theorem \ref{theorem:lower-bound-sto-consumption-multiple-resource} concerns spending all the resource on pulling a single arm. If the consumption follows Bernoulli Distribution, the total pulling times of this single arm cannot be as large as the deterministic consumption setting. And this difference also contributes to the operator $\max_{j\in \{1,i\}}$ in Theorem \ref{theorem:lower-bound-sto-consumption-multiple-resource}, diverging from the conventional form of $\max_{j\in [K]}$. The idea of allocating all the resource to arm $i$, somehow replace the step 2 in appendix \ref{pf:thm_low_det}. This suggests that there is room for improvement in the approximation. A tighter approximation in the exponential term is to be explored.

At the end of this subsection, we present the sketch proof of Theorem \ref{theorem:lower-bound-sto-consumption-multiple-resource}.
\begin{proof}{Sketch Proof of Theorem \ref{theorem:lower-bound-sto-consumption-multiple-resource}}
    Detailed proof is in Appendix \ref{pf:thm_low_sto}. The overall idea is still based on Transportation Equality in \cite{CarpentierL16}, but we modify part of the steps. To find an appropriate threshold $\bar{c}$, we require $\bar{c}$ is small enough, such that for any $c\in (0, \bar{c})$, $d_{\ell, (j)}= cd^{0}_{\ell, (j)}, \forall j\in [K], \forall \ell\in[L]$, we have inequalities such as $128\left(\frac{g(d_{\ell, (1)})}{(r_1-r_2)^2} + \sum_{k=3}^K \frac{g(d_{\ell, (k)})}{(r_1-r_k)^2}\right)\log\frac{1}{d_{\ell, (i)}} < 1,\forall \ell\in [L]$

    \textbf{Step 1.} We show that, 
    \begin{align*}
        & \Pr_i(\psi\neq i)\\
        \geq & \Pr_{1}\left((\psi \neq i) \text{ and } (T_i\leq \bar{T}_{i,\ell_0}) \text{ and }\xi_{i,\ell_0}\right)\\
        & \exp\left(-\frac{3}{4}\frac{C_{\ell_0}}{\tilde{H}^{\text{sto}}_{1, \ell_0}}\right).
    \end{align*}
    holds for all $\ell_0\in [L]$ and $i\in \{2,3,4,\cdots,K\}$,
    where $T_i = \sum_{t=1}^{T} \mathbbm{1}(A_t =i)$, $\bar{T}_{i,\ell_0}:=\frac{C_{\ell_0}}{\frac{g(d_{\ell_0,(1)})}{(r_1-r_2)^2} + \sum_{k=3}^K \frac{g(d_{\ell_0,(k)})}{(r_1-r_k)^2}}\frac{1}{(r_1-r_i)^2}=\frac{C_{\ell_0}}{\tilde{H}^{\text{sto}}_{1,\ell_0}}\frac{1}{(r_1-r_i)^2}$, and 
    \begin{align*}
        & \xi_{i,\ell_0}\\
        = & \Bigg\{\max_{1\leq t\leq \bar{T}_{i,\ell_0}} \sum_{s=1}^{t}\log\frac{d\nu_i^{(1)}}{d\nu_i^{(i)}}(\tilde{R}_i(s))-\frac{(1-2r_i)^2}{2}\\
        & \leq I_{i,\ell_0}\Bigg\}
    \end{align*}
    is the concentration event for the realized KL divergence, similar to the proof of Theorem \ref{theorem:lower-bound-fixed-consumption-multiple-resource}.

    \textbf{Step 2.} We assume $\Pr_1\left(\psi = i\right)\leq \exp\left(-\min_{\ell \in [L]}\frac{C_\ell}{\tilde{H}^{\text{sto}}_{1, \ell}}\right)$. If this assumption doesn't hold, we have already proved theorem \ref{theorem:lower-bound-sto-consumption-multiple-resource}. In addition, we will use anti concentration result of Bernoulli Distribution, i.e Lemma \ref{lemma:KL-Divergence-Chernoff-ln} to prove $\Pr_{1}\left(T_i \leq \bar{T}_{i,\ell_0}\right)\geq \exp\left(-\frac{C_{\ell_0}}{\frac{1}{4\log\frac{1}{d_{\ell_0, (i)}}}}\right)$. Combining all the results, we have for all $\ell_0\in [L]$,
    \begin{align*}
        & \Pr_{1}\left((\psi \neq i) \text{ and } (T_i \leq \bar{T}_{i,\ell_0}) \text{ and }\xi_{i,\ell_0}\right)\\
        \geq & \exp\left(-\frac{C_{\ell_0}}{\frac{1}{4\log\frac{1}{d_{\ell_0, (i)}}}}\right) - \exp\left(-\frac{C_{\ell_0}}{32\tilde{H}^{\text{sto}}_{1, \ell_0}}\right)\\
        & - \exp\left(-\frac{C_{\ell_0}}{\tilde{H}^{\text{sto}}_{1, \ell_0}}\right),
    \end{align*}
    which further requires the inequality $\Pr_{1}\left(\neg \xi_{i,\ell_0}\right)\leq \exp\left(-\frac{C_{\ell_0}}{32\tilde{H}^{\text{sto}}_{1, \ell_0}}\right)$. Similar to the proof of Theorem \ref{theorem:lower-bound-fixed-consumption-multiple-resource}, we need to use Lemma \ref{lemma:bound-realized-KL} to prove this inequality.

    \textbf{Step 3.} From the step 1 and step 2, we have concluded $\Pr_i(\psi\neq i) \geq \left(\exp\left(-\frac{C_{\ell_0}}{\frac{1}{4\log\frac{1}{d_{\ell_0, (i)}}}}\right) - \exp\left(-\frac{C_{\ell_0}}{32\tilde{H}^{\text{sto}}_{1, \ell_0}}\right) - \exp\left(-\frac{C_{\ell_0}}{\tilde{H}^{\text{sto}}_{1, \ell_0}}\right)\right)\exp\left(-\frac{3}{4}\frac{C_{\ell_0}}{\tilde{H}^{\text{sto}}_{1, \ell_0}}\right)$.  
    Use the property of $\bar{c}$, we can conclude
    \begin{align*}
        \exp\left(-\frac{C_{\ell_0}}{\frac{1}{4\log\frac{1}{d_{\ell_0, (i)}}}}\right) - \exp\left(-\frac{C_{\ell_0}}{32\tilde{H}^{\text{sto}}_{1, \ell_0}}\right) - \\\exp\left(-\frac{C_{\ell_0}}{\tilde{H}^{\text{sto}}_{1, \ell_0}}\right) \geq \exp\left(-\frac{C_{\ell_0}}{4\tilde{H}^{\text{sto}}_{1, \ell_0}}\right).
    \end{align*}
    And we complete the proof.
\end{proof}

\section{Numeric Experiments}
\label{sec:Numeric-Experiment}
We conducted a performance evaluation of the SH-RR method on both synthetic and real-world problem sets. Our evaluation included a comparison of SH-RR against four established baseline strategies: Anytime-LUCB (AT-LUCB) \citep{jun_anytime_2016}, Upper Confidence Bounds (UCB) \citep{BubeckMS09}, Uniform Sampling, and Sequential Halving \citep{KarninKS13} augmented with the doubling trick. Unlike fixed confidence and fixed budget strategies, these baseline methods are \emph{anytime} algorithms, which recommend an arm as the best arm after each arm pull, continuously, until a specified resource constraint is violated. The evaluation was carried out until a resource constraint was breached, at which point the last recommended arm was returned as the identified arm. Fixed confidence and fixed budget strategies were deemed inapplicable to the BAIwRC problem as they necessitate an upper bound on the BAI failure probability and an upper limit on the number of arm pulls in their respective settings. Further details regarding the experimental set-ups are elaborated in appendix \ref{sec:Details-on-the-Numerical-Experiment-Set-ups}.

\begin{figure}
    \centering
    \begin{subfigure}
      \centering
      \includegraphics[width=0.5\textwidth]{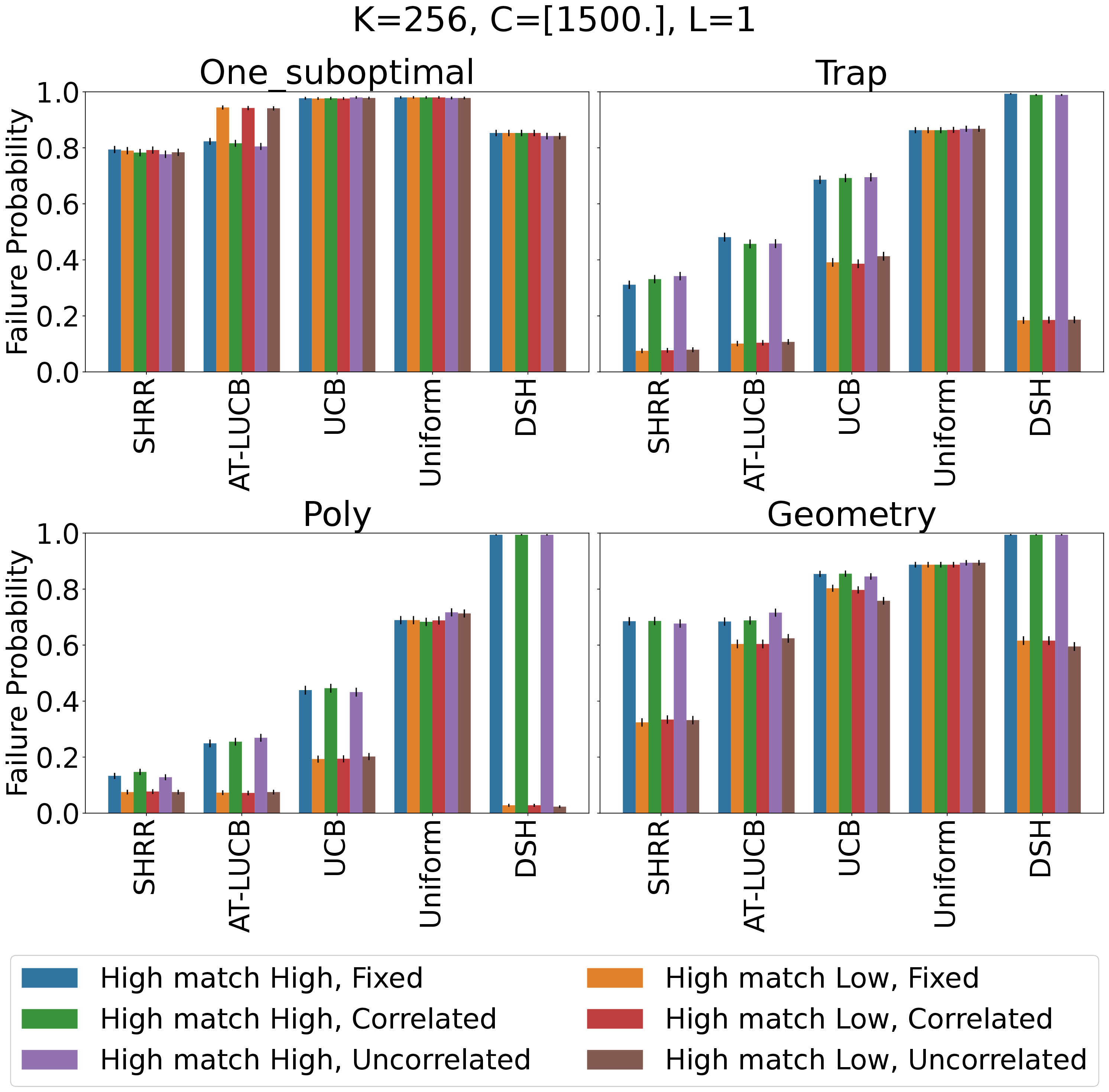}
    \end{subfigure}%
    \begin{subfigure}
      \centering
      \includegraphics[width=0.5\textwidth]{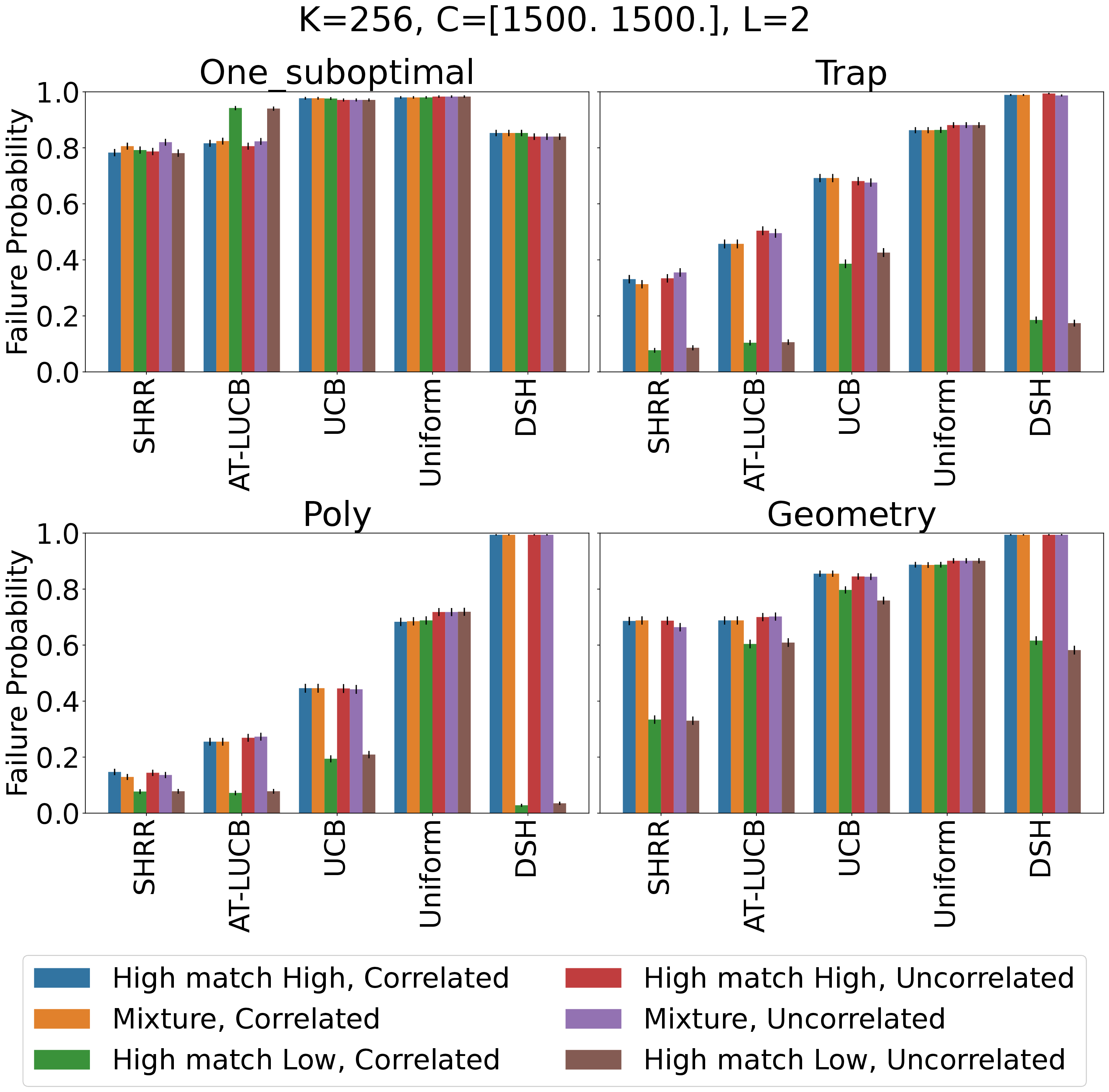}
    \end{subfigure}
    \caption{Comparison of SH-RR and anytime baselines in different setups}
    \label{fig:numeric-result}
\end{figure}

\textbf{Synthesis problems.}
We investigated the performance of our algorithm across various synthetic settings, each with distinct reward and consumption dynamics. (1) \emph{High match High} (\textsf{HmH}) where higher mean rewards correspond to higher mean consumption, (2) \emph{High match Low} (\textsf{HmL}) where they correspond to lower mean consumption, and (3) \emph{Mixture} (\textsf{M}) where each arm consumes less of one resource while consuming more of another, applicable when $L=2$. Additionally, resource consumption variability was categorized into deterministic, correlated (random and correlated with rewards), and uncorrelated (random but independent of rewards), with the deterministic setting omitted for $L=2$ due to similar results to $L=1$. Reward variability across arms was explored through four settings: One Group of Sub-optimal, Trap, Polynomial, and Geometric, analogous to \cite{KarninKS13}. The various combinations of these settings are illustrated in Figure \ref{fig:numeric-result}, with more detailed descriptions provided in Appendices \ref{sec:Details-on-the-Numerical-Experiment-Set-ups-single} and \ref{sec:Details-on-the-Numerical-Experiment-Set-ups-multiple}.

Figure \ref{fig:numeric-result} presents the failure probability of the different strategies in different setups, under {\color{black} $K=256$, $L=1, 2$ with an initial budget of 1500 for each resource}. Each strategy was executed over 1000 independent trials, with the failure probability quantified as {\color{black} $(\text{\# trials that fails BAI}) / 1000$}. Our analysis anticipated a higher difficulty level for the \textsf{HmH} instances, a notion substantiated by the experimental outcomes. The bottom panel illustrates comparable performances on \textsf{M} and \textsf{HmH} setups, suggesting the scarcer resource could predominantly influence the performance. This observed behavior aligns with the utilization of the 
min operator in the definitions of $\gamma^{\text{det}}$ and $\gamma^{\text{sto}}$.

Our proposed algorithm SH-RR is competitive compared to these state-of-the-art benchmarks. SH-RR achieves the best performance by a considerable margin for \textsf{HmL}, while still achieve at least a matching performance compared to the baselines for \textsf{HmH}. SH-RR favors arms with relatively high empirical reward, thus when those arms consume less resources, SH-RR can achieve a higher probability of BAI. For confidence bound based algorithms such as AT-LUCB and UCB, resource-consumption-heavy sub-optimal arms are repeatedly pulled, leading to resource wastage and a higher failure probability. These empirical results demonstrate the necessity of SH-RR algorithm for BAIwRC in order to achieve competitive performance in a variety of settings.  

\textbf{Real-world problems.} We implemented different machine learning models, each adorned with various hyperparameter combinations, as distinct arms. The overarching goal is to employ diverse BAI algorithms to unravel the most efficacious model and hyperparameter ensemble for tackling supervised learning tasks. There is a single constraint on running time for each BAI experiment. To meld simplicity with time-efficiency, we orchestrated implementations of four quintessential, yet straightforward machine learning models: K-Nearest Neighbour, Logistic Regression, Adaboost, and Random Forest. Each model is explored with eight unique hyperparameter configurations. We considered 5 classification tasks, including (1) Classify labels 3 and 8 in part of the MNIST dataset (MNIST 3\&8). (2) Optical recognition of handwritten digits data set (Handwritten). (3) Classify labels -1 and 1 in the MADELON dataset (MADELON). (4) Classify labels -1 and 1 in the Arcene dataset (Arcene). (5) Classify labels on weight conditions in the Obesity dataset (Obesity). See appendix \ref{sec:details-realworld} for details on the set-up.

We designated the arm with the lowest empirical mean cross-entropy, derived from a combination of machine learning models and hyperparameters, as the best arm. Our BAI experiments were conducted across 100 independent trials. During each arm pull in a BAI experiment round—i.e., selecting a machine learning model with a specific hyperparameter combination—we partitioned the datasets randomly into training and testing subsets, maintaining a testing fraction of 0.3. The training subset was utilized to train the machine learning models, and the cross-entropy computed on the testing subset served as the realized reward. 

The results, showcased in Table \ref{tab: Performance-of-BAI-Real-World} and \ref{tab: Performance-of-BAI-Real-World-cont}, delineate the failure probability in identifying the optimal machine learning model and hyperparameter configuration for each BAI algorithm. Amongst the tested algorithms, SH-RR emerged as the superior performer across all experiments. This superior performance can be attributed to two primary factors: (1) classifiers with lower time consumption, such as KNN and Random Forest, yielded lower cross-entropy, mirroring the \textsf{HmL} setting; and (2) the scant randomness in realized Cross-Entropy ensured that after each half-elimination in SH-RR, the best arm was retained, underscoring the algorithm's efficacy.

\begin{table}[h]
\begin{center}
    \caption{Failure Probability of different BAI strategies on Real-life datasets}
    \label{tab: Performance-of-BAI-Real-World}
    \centering
    \begin{tabular}{ cccccc } 
      \toprule
      Algorithm & MNIST 3\&8 & Handwritten
      \\
      \midrule
      SHRR & \textbf{0} & \textbf{0.12}\\
      ATLUCB & 0.21 & 0.23\\  
      UCB & 0.21 & 0.34\\
      Uniform & 0.21  & 0.25\\ 
      DSH & 0.14 &  0.20\\
      \bottomrule
    \end{tabular}
    \end{center}
\end{table}

\begin{table}[h]
\begin{center}
    \caption{Failure Probability of different BAI strategies on Real-life datasets, cont.}
    \label{tab: Performance-of-BAI-Real-World-cont}
    \centering
    \begin{tabular}{ cccccc } 
      \toprule
      Algorithm & Arcene & Obesity  & MADELON
      \\
      \midrule
      SHRR & \textbf{0.38} & \textbf{0.31}  & \textbf{0}\\
      ATLUCB & 0.6 & 0.43  & 0.42\\  
      UCB & 0.71 & 0.43  & 0.30 \\
      Uniform & 0.81 & 0.56   & 0.29\\ 
      DSH & 0.67 & 0.54  & 0.12\\
      \bottomrule
    \end{tabular}
    \end{center}
\end{table}

\section{Conclusion}\label{sec:Conclusion}
We follow the problem formulation and improve the theoretical conclusions in \cite{anonymous2024best}. In the analysis of the upper bound of the failure probability, we define a new structure $H_{2, \ell}(Q)$ that can fit in the nature of both deterministic and stochastic consumption setting. The unified analysis suggests how the resource consumption distribution affect the difficulty of the BAIwRC problem, also indicate the excellency of proposed algorithm SH-RR, in the sense of handling all the BAIwRC problems. In the analysis of the lower bound of the failure probability, we improve the usage of Transportation Equality in \cite{CarpentierL16}, resulting in more relaxing assumptions on the required minimum value of budget $\{C_{\ell}\}_{\ell=1}^L$.

In summary, these extra refinements imply both the versatility and robustness of the proposed approach.

\bibliographystyle{abbrvnat}
\bibliography{arxiv-submission/reference}

\newpage

\appendix

\section{Auxiliary Results}
\begin{lemma}[Bennett]
    \label{lemma:Bennett}
    Let $X_1,X_2,\cdots,X_n$ be independent random variables with finite variance. Take $S_n=\sum_{i=1}^n (X_i-\mathbb{E}X_i)$, $\sigma^2=\text{Var}(S_n)$, $|X_i-\mathbb{E}X_i|\leq a$, $h(u)=(u+1)\log(u+1)-u$, we have
    \begin{align*}
        \Pr\left(S_n > t\right)\leq & \exp\left(-\frac{\sigma^2}{a^2}h(\frac{at}{\sigma^2})\right).
    \end{align*}
\end{lemma}
Lemma \ref{lemma:Bennett} can be found in the exercise 2.2 of \cite{devroye2001combinatorial}.

Based on Lemma \ref{lemma:Bennett}, we can easily derive the following corollary
\begin{corollary}
    \label{corollary:Bennett}
    Let $X_1,X_2,\cdots,X_n$ be independent random variables with finite variance. Take $S_n=\sum_{i=1}^n (X_i-\mathbb{E}X_i)$, $\sigma^2=\text{Var}(S_n)$, $|X_i-\mathbb{E}X_i|\leq a$, we have 
    \begin{align*}
        \Pr\left(S_n > t\right)\leq & \exp\left(-\frac{t\log (\frac{at}{\sigma^2}+1)}{2a}\right).
    \end{align*}
\end{corollary}
\begin{proof}{Proof of Corollary \ref{corollary:Bennett}}
    Suffice to prove $(u+1)\log(u+1)-u\geq \frac{u\log(u+1)}{2}$ holds for any $u\geq 0$. If this is true, we have
    \begin{align*}
        & \exp\left(-\frac{\sigma^2}{a^2}h(\frac{at}{\sigma^2})\right)\\
        \leq & \exp\left(-\frac{\sigma^2}{a^2}\frac{\frac{at}{\sigma^2}\log (\frac{at}{\sigma^2}+1)}{2}\right)\\
        = & \exp\left(-\frac{t\log (\frac{at}{\sigma^2}+1)}{2a}\right).
    \end{align*}
    By the Bennett's Lemma, the proof is done.

    Then we turn to prove $(u+1)\log(u+1)-u\geq \frac{u\log(u+1)}{2}$ holds for any $u\geq 0$. Define $\psi(u)=(u+1)\log(u+1)-u-\frac{u\log(u+1)}{2}$, easy to see
    \begin{align*}
        & \frac{d\psi}{du}\\
        = & 1 + \log(u+1) - 1 -\frac{1}{2}\left(\frac{u}{u+1}+\log(u+1)\right)\\
        = & \frac{1}{2}\log(u+1)-\frac{1}{2}\frac{u}{u+1}\\
        = & \frac{1}{2}\left(\log(u+1)+\frac{1}{u+1}-1\right).
    \end{align*}
    Easy to validate $\log(u+1)+\frac{1}{u+1}-1\geq 0$ holds for all $u>0$, indicating that $\psi(u)$ is increasing in the interval $[0, +\infty)$. Since $\psi(0)=0$, we can conclude $(u+1)\log(u+1)-u\geq \frac{u\log(u+1)}{2}$.
\end{proof}

\section{Proofs}
\subsection{Proof of Claim \ref{claim:feasibility}}\label{app:claim_feasibility}
\begin{proof}{Proof of Claim \ref{claim:feasibility}}
    The total type-$\ell$ resource consumption is
    \begin{align*}
        &\sum^{\lceil \log_2 K\rceil -1}_{q = 0} I^{(q)}_\ell\nonumber\\
         =  &\sum^{\lceil \log_2 K\rceil -1}_{q = 0} \left[ \frac{C_\ell}{\lceil\log_2 K\rceil} + ( \textsf{Ration}^{(q)}_\ell- \textsf{Ration}^{(q+1)}_\ell  )\right]\\
        = & C_\ell + ( \textsf{Ration}^{(0)}_\ell- \textsf{Ration}^{(\lceil\log_2 K\rceil)}_\ell  ).
    \end{align*}
    We complete the proof by showing that  $\textsf{Ration}^{(0)}_\ell = \frac{C_\ell}{\lceil\log_2 K\rceil}\leq \textsf{Ration}^{(q)}_\ell$ with certainty for every $q$. Indeed, with certainty we have $I^{(q-1)}_\ell \leq \textsf{Ration}^{(q-1)}_\ell$ for every $q \geq 1$. The \textbf{while} loop maintains that $I^{(q - 1)}\leq \textsf{Ration}^{(q-1)}_\ell - 1$, which ensures that $I^{(q - 1)}\leq \textsf{Ration}^{(q-1)}_\ell$ when the \textbf{while} loop ends, and consequently $\frac{C_\ell}{\lceil\log_2 K\rceil}\leq \textsf{Ration}^{(q)}_\ell$ by Line \ref{alg:ration}. Altogether, the claim is shown.
\end{proof}

\subsection{Proof of Theorem \ref{theorem:upper-bound-of-failure} and Theorem \ref{theorem:upper-bound-of-failure-det}}
\label{appendix:theorem_upper}
In the case that the mean consumptions are deterministic, we can improve the Theorem \ref{theorem:upper-bound-of-failure} by removing the factor $L$, which is the following.
\begin{theorem}
    \label{theorem:upper-bound-of-failure-det}
    Consider a BAIwRC instance $Q$ in the deterministic consumption setting, i.e $\sigma=0$. The SH-RR algorithm has BAI failure probability $\Pr(\psi \neq 1)$ at most
    \begin{align}\label{eq:upper-bound-of-failure-det}
        K (\log_2 K) \exp\left(-\frac{1}{4}\min_{\ell \in[L]} \{ \frac{C_{\ell}}{\lceil\log_2 K\rceil H_{2, \ell}^{\text{det}}}\}\right)
    \end{align}
    where  $\gamma^{\text{det}}(Q) = \min_{\ell\in [L]}\{C_\ell / H_{2, \ell}^{\text{det}}(Q)\}$, and $H_{2, \ell}^{\text{det}}(Q)$ is defined in (\ref{eq:det_H}).
\end{theorem}

Before the proof, we need a lemma to bound the pulling times of arms. 
\begin{lemma}\label{lemma:bound_count}
For $a=1,2,\cdots,K$, $K\geq 1$, let $\{X_{a,s}\}_{s=1}^{\infty}$ be i.i.d random variable following $\nu_a$. Assume $\nu_a\in\mathscr{C}_{b, \sigma^2}=\left\{\nu:\text{supp}(\nu)\subset [0, 1],\Pr_{D\sim \nu}(|D-\mathbb{E}D|\leq b)=1,\text{Var}(\nu)\leq \sigma^2\right\}$ holds for all $a\in[K]$. For $a=1,2,\cdots,K$, denote $\mathbb{E}X_{a,s}=d_a$. Then for any positive value $C$, denote $T=\frac{C}{\sum_{a=1}^K f(b,\sigma, d_a)}$, where the function $f$ is defined in (\ref{eq:def_f}).   we have
\begin{align*}
    \Pr\left(\sum_{s=1}^{T} \sum_{a=1}^K X_{a,s} > C\right)
    \leq \exp\left(-T\right),
\end{align*}
\end{lemma}
Lemma \ref{lemma:bound_count} assumes $T$ is an integer, as the rational issue will only lead to an extra constant factor.

Denote $T^{(q)}$ as the pulling times of each survival arm at phase $q$, and $\bar{T}^{(q)} = T^{(1)} + \ldots + T^{(q)}$ as the total pulling times of each survival arm at the end of phase $q$. Define $E_q:=\left\{i: \hat{r}_{i,\bar{T}^{(q)}}> \hat{r}_{1, \bar{T}^{(q)} }\right\}$ and the bad event
\begin{equation}
B^{(q)}=\left\{ |E_q| \ge \lceil\frac{K}{2^{q+1}}\rceil\right\}.
\end{equation}
We assert that, for any phase $q$, 
\begin{align}
    &\Pr\left(B^{(q)}\right)\leq  2LK \exp\left(-\frac{1}{4}\min_{\ell \in[L]} \{ \frac{C_{\ell}}{\lceil\log_2 K\rceil H_{2, \ell}^{}}\}\right).\label{eq:upper_main_assert}
\end{align}
Remark $\{\hat{k}=1\}\supset \cap_{q=1}^{\log_2 K} \neg B^{(q)}$. Once (\ref{eq:upper_main_assert}) is shown, the Theorem can be proved by a union bound over all phases:
\begin{align*}
    &\Pr(\hat{k}\neq 1)\\
    \le &\Pr(\cup_{q=1}^{\log_2 K} B^{(q)})\\
    \le &\sum_{q=1}^{\log_2 K} \Pr(B^{(q)})\\
    \le &2LK (\log_2 K) \exp\left(-\frac{1}{4}\min_{\ell \in[L]} \{ \frac{C_{\ell}}{\lceil\log_2 K\rceil H_{2, \ell}^{}}\}\right)
\end{align*}
In what follows, we establish the main claim (\ref{eq:upper_main_assert}). For our analysis, we define
$\beta^{(q)}_\ell:=\frac{C_{\ell}}{ \lceil \log_2 K\rceil\left(\sum_{k=1}^{\lceil \frac{K}{2^q} \rceil} f(b_{\ell}, \sigma_{\ell}, d_{\ell, (k)}) \right)}$ 
and $\bar{\beta}^{(q)} = \min_{\ell\in [L]}\{\beta^{(q)}_\ell\}$, then we can split $\Pr\left(B^{(q)}\right)$ into two parts.
\begin{align*}
    &\Pr(B^{(q)})\\
    \le & \Pr\left(\exists k \ge \lceil\frac{K}{2^{q+1}}\rceil, \hat{r}_{k, \bar{T}^{(q)}} > \hat{r}_{1, \bar{T}^{(q)}}\right)\\
    \le &\underbrace{ \Pr\left(\exists k \ge \lceil\frac{K}{2^{q+1}}\rceil, \hat{r}_{k, \bar{T}^{(q)}} > \hat{r}_{1, \bar{T}^{(q)}}, \bar{T}^{(q)} \ge \bar{\beta}^{(q)} \right) }_{(\P)}\\
    &\qquad \qquad \qquad \qquad+ \underbrace{ \Pr\left(\bar{T}^{(q)} < \bar{\beta}^{(q)}\right) }_{(\ddagger)}
\end{align*}
To facilitate our discussions, we denote $\{\tilde{D}^{(q)}_{\ell, k}(n)\}^\infty_{n=1}$ as i.i.d. samples of the random consumption of resource $\ell$ by pulling arm $k$. For the term $(\ddagger)$, we have
\begin{align}
    & \Pr\left(\bar{T}^{(q)} < \bar{\beta}^{(q)} \mid \tilde{S}^{(q)} \right)\nonumber \\
    \leq & \Pr\left(T^{(q)} < \bar{\beta}^{(q)} \mid \tilde{S}^{(q)} \right)\nonumber\\
    \leq & \Pr\Bigg(\exists \ell\in [L], \sum_{n=1}^{\beta^{(q)}_\ell}\sum_{k\in \tilde{S}^{(q)}} \tilde{D}^{(q)}_{\ell, k}(n) \geq \frac{C_{\ell}}{\lceil \log_2 K\rceil}\mid \tilde{S}^{(q)}\Bigg)\label{eq:by_ration}\\
    \leq & \sum_{\ell=1}^L \Pr\Bigg(\sum_{n=1}^{\beta^{(q)}_\ell}\sum_{k\in \tilde{S}^{(q)}} \tilde{D}^{(q)}_{\ell, k}(n) \geq \frac{C_{\ell}}{\lceil \log_2 K\rceil}\mid \tilde{S}^{(q)}\Bigg)\nonumber\\
    \leq & \sum_{\ell=1}^L \Pr\Bigg(\sum_{n=1}^{\frac{C_\ell}{ \lceil \log_2 K\rceil \cdot \sum_{k\in \tilde{S}^{(q)}} f(b_{\ell},\sigma_{\ell},d_{\ell, k})}}\sum_{k\in \tilde{S}^{(q)}} \tilde{D}^{(q)}_{\ell, k}(n) \geq \frac{C_{\ell}}{\lceil \log_2 K\rceil}\mid \tilde{S}^{(q)}\Bigg)\nonumber\\
    \leq & \sum_{\ell=1}^L\exp\left(-\frac{C_\ell}{ \lceil \log_2 K\rceil \cdot \sum_{k\in \tilde{S}^{(q)}} f(b_{\ell},\sigma_{\ell},d_{\ell, k})} \right)\label{eq:by_lemma_4_1}\\
    \leq & L\exp\left(-\bar{\beta}^{(q)}\right) \nonumber
\end{align}
Step (\ref{eq:by_ration}) is by the invariance $\textsf{Ration}_\ell^{(q)} \geq \frac{C_\ell}{\lceil \log_2 K\rceil}$ maintained by the \textbf{while} loop of SH-RR. Step (\ref{eq:by_lemma_4_1}) is by applying Lemma \ref{lemma:bound_count}. 

Next, we analyze the term $(\P)$, we denote $$(\dagger)^{(q)}_k=\Pr\left(\hat{r}_{1, \bar{T}^{(q)}}^{(q)}<\hat{r}_{k, \bar{T}^{(q)}}^{(q)},\bar{T}^{(q)}>\bar{\beta}^{(q)}\right).$$ 
We remark that $(\P) \leq  \sum_{k= \lceil\frac{K}{2^{q+1}}\rceil}^K (\dagger)^{(q)}.$ Let $\{R^{(q)}_{k}(n)\}^\infty_{n=1},\{R^{(q)}_{1,n}\}^\infty_{n=1}$ be i.i.d. samples of the rewards under arm $k$ and arm 1 respectively. For each $n$, we define $W(n)=R^{(q)}_{k}(n)-R^{(q)}_{1}(n)+\Delta_k$, where we recall that $\Delta_k = r_1 - r_k$. Clearly, $\mathbb{E}[W(n)]=0$, and $W(n)$ are i.i.d. 1-sub-Gaussian. For any $\lambda> 0$, we have
\begin{align}
    &(\dagger)^{(q)}_k \nonumber\\
    \leq &\Pr\left(\exists N \geq \bar{\beta}^{(q)}, \frac{\sum_{n=1}^N G_n}{N}>\Delta_k\right)\nonumber\\
    =&\Pr\left(\sup_{N\ge \bar{\beta}^{(q)}} \frac{\exp(\lambda \sum_{n=1}^N G_n)}{\exp(N\lambda \Delta_k)} > 1\right) \nonumber\\
    \le& \mathbb{E}\left[\frac{\exp\left(\lambda\sum_{n=1}^{\lceil \bar{\beta}^{(q)}\rceil } W(n)\right)}{\exp\left(\lambda \lceil \bar{\beta}^{(q)}\rceil \Delta_k\right)}  \right]\label{eq:doob}\\
    \leq &\frac{\exp\left(\frac{\lambda^2 \lceil \bar{\beta}^{(q)}\rceil}{2}\right)}{\exp\left(\lambda \lceil \bar{\beta}^{(q)}\rceil \Delta_k\right)}\label{eq:rew-subgau}.
\end{align}
Step (\ref{eq:doob}) is by the maximal inequality for (super)-martingale, which is a Corollary of the Doob's optional stopping Theorem, see Theorem 3.9 in \cite{LattimoreS2020} for example. Step (\ref{eq:rew-subgau}) is by the fact that $G(n)$ is 1-sub-Gaussian. Finally, applying $\lambda = \Delta_k$, (\ref{eq:rew-subgau}) leads us to
$
(\dagger)^{(q)}_k \leq \exp(- \bar{\beta}^{(q)}\Delta_k^2 / 2), 
$, meaning
\begin{align}
(\P) &\leq K \exp\left(-\bar{\beta}^{(q)} \frac{(r_1-r_{ \lceil\frac{K}{2^{q+1}}\rceil })^2}{2}\right)\label{eq:final_P}.
\end{align}
Step (\ref{eq:final_P}) is by the assumption that $\Delta_k$ is not decreasing. Altogether, combining the upper bounds (\ref{eq:by_lemma_4_1},\ref{eq:final_P}) to $(\ddagger) , (\P)$ respectively, leads us to the proof of (\ref{eq:upper_main_assert}).
\begin{align*}
    &\Pr(B^{(q)})\\
    \le & \sum_{k= \lceil\frac{K}{2^{q+1}}\rceil}^K \Pr\left(\hat{r}_{k, \tilde{T}_q} > \hat{r}_{1, \tilde{T}_q}, \bar{T}_q \ge \bar{\beta}^{(q)}\right)+L\exp\left(-\bar{\beta}^{(q)}\right) \\
    \leq & K \exp\left(-\bar{\beta}^{(q)}\cdot \frac{(r_1-r_{\lceil\frac{K}{2^{q+1}}\rceil})^2}{2}\right) + L\exp\left(-\bar{\beta}^{(q)}\right)\\
    \leq & 2LK \exp\left(-\frac{1}{2}\min_{\ell \in[L]} \{ \frac{C_{\ell}}{2\lceil\log_2 K\rceil H_{2, \ell}^{}}\}\right)\\
    = & 2LK \exp\left(-\frac{1}{4}\min_{\ell \in[L]} \{ \frac{C_{\ell}}{\lceil\log_2 K\rceil H_{2, \ell}^{}}\}\right).
\end{align*}

To prove Theorem \ref{theorem:upper-bound-of-failure-det}, we follow the similar idea. Recall that when the consumptions are deterministic, we have
$\beta^{(q)}_\ell:=\frac{C_{\ell}}{ \lceil \log_2 K\rceil\left(\sum_{k=1}^{\lceil \frac{K}{2^q} \rceil} f(b_{\ell}, \sigma_{\ell}, d_{\ell, (k)}) \right)}=\frac{C_{\ell}}{ \lceil \log_2 K\rceil \sum_{k=1}^{\lceil \frac{K}{2^q} \rceil} d_{\ell, (k)}}$. Then, we can conclude $\Pr\left(\bar{T}^{(q)} < \bar{\beta}^{(q)}\right)=0$. Following the same steps as above, we have
\begin{align*}
    & \Pr(B^{(q)})
    \le \Pr\left(\exists k \ge \lceil\frac{K}{2^{q+1}}\rceil, \hat{r}_{k, \bar{T}^{(q)}} > \hat{r}_{1, \bar{T}^{(q)}}, \bar{T}^{(q)} \ge \bar{\beta}^{(q)} \right)\\
    & \Pr\left(\hat{r}_{1, \bar{T}^{(q)}}^{(q)}<\hat{r}_{k, \bar{T}^{(q)}}^{(q)},\bar{T}^{(q)}>\bar{\beta}^{(q)}\right)\leq  \exp(- \bar{\beta}^{(q)}\Delta_k^2 / 2).
\end{align*}
Combining the conclusions, we have
\begin{align*}
    &\Pr(\hat{k}\neq 1)\\
    \le &\Pr(\cup_{q=1}^{\log_2 K} B^{(q)})\\
    \le &\sum_{q=1}^{\log_2 K} \Pr(B^{(q)})\\
    \le & K (\log_2 K) \exp\left(-\frac{1}{4}\min_{\ell \in[L]} \{ \frac{C_{\ell}}{\lceil\log_2 K\rceil H_{2, \ell}^{\text{det}}}\}\right),
\end{align*}
meaning that we have proved Theorem \ref{theorem:upper-bound-of-failure-det}.

\subsection{Proof of Lemma \ref{lemma:bound_count}}
\begin{proof}{Proof of Lemma \ref{lemma:bound_count}}
    Apply some simple calculation, we have
    \begin{align*}
         & \Pr\left(\sum_{s=1}^T \sum_{a=1}^K X_{a,s} > C\right)\\
         = & \Pr\left(\sum_{s=1}^T\sum_{a=1}^K (X_{a,s}-d_a) > C-T\sum_{a=1}^Kd_a\right).
    \end{align*}
    By the Corollary \ref{corollary:Bennett}, we have
    \begin{align*}
        & \Pr\left(\sum_{s=1}^T \sum_{a=1}^K X_{a,s} > C\right)\\
        \leq & \exp\left(-\frac{C-T\sum_{a=1}^Kd_a}{b}\cdot \frac{\log(\frac{b(C-T\sum_{a=1}^Kd_a)}{TK\sigma^2}+1)}{2}\right)\\
        = & \exp\Bigg(-2TK (\frac{C}{4bKT}-\frac{\sum_{a=1}^Kd_a}{4 K b})\log(\frac{\frac{C}{4bKT}-\frac{\sum_{a=1}^Kd_a}{4bK}}{\frac{\sigma^2}{4b^2}}+1)\Bigg).
    \end{align*}
    Recall $T=\frac{C}{\frac{4Kb}{\log (\frac{4b^2}{\sigma^2}+1)}+\sum_{a=1}^Kd_a}\Leftrightarrow \frac{C}{4bKT}-\frac{\sum_{a=1}^Kd_a}{4 K b}=\frac{1}{\log (\frac{4b^2}{\sigma^2}+1)}$, we can conclude
    \begin{align*}
        & K (\frac{C}{4bKT}-\frac{\sum_{a=1}^Kd_a}{4 K b})\cdot \log(\frac{\frac{C}{4bKT}-\frac{\sum_{a=1}^Kd_a}{4bK}}{\frac{\sigma^2}{4b^2}}+1)\\
        = & K \cdot \frac{1}{\log (\frac{4b^2}{\sigma^2}+1)} \cdot \log\left(\frac{\frac{1}{\log (\frac{4b^2}{\sigma^2}+1)}}{\frac{\sigma^2}{4b^2}}+ 1\right)\\
        \geq & \frac{K}{2}.
    \end{align*}
    The last step is by the following facts. By the definition, we have $\frac{4b^2}{\sigma^2}\geq 4$. And $x\geq 4$ guarantees $\frac{\log\left(\frac{x}{\log(x+1)}+1\right)}{\log(x+1)}\geq \frac{\log\left(\frac{x}{\log(x+1)}+1\right)}{\log(x+\log(x+1))}=1-\frac{\log \log(x+1)}{\log(x+\log(x+1))}$. Define $\phi(x)=x+\log (x+1)-(\log (x+1))^2$. Easy to see $\frac{d\phi}{dx}=\frac{(x+1)+1-2\log(x+1)}{x+1} > 0$ holds for all $x\geq 4$, and $\phi(4)>0$. We can conclude $x+\log (x+1)\geq (\log (x+1))^2$ holds for all $x\geq 4$. Notice that
    \begin{align*}
        & (\log (x+1))^2 \leq x + \log(x+1)\\
        \Leftrightarrow & \log (x+1)\leq \sqrt{x+\log(x+1)}\\
        \Leftrightarrow & \log \log(x+1)\leq \frac{\log(x+\log(x+1))}{2}\\
        \Leftrightarrow & \frac{\log \log(x+1)}{\log(x+\log(x+1))}\leq \frac{1}{2}.
    \end{align*}
    We can conclude $K \cdot \frac{1}{\log (\frac{4b^2}{\sigma^2}+1)} \cdot \log\left(\frac{\frac{1}{\log (\frac{4b^2}{\sigma^2}+1)}}{\frac{\sigma^2}{4b^2}}+ 1\right)\geq \frac{K}{2}$, and further
    \begin{align*}
        \Pr\left(\sum_{s=1}^T \sum_{a=1}^K X_{a,s} > C\right)\leq \exp(-TK)\leq \exp(-T).
    \end{align*}
\end{proof}

\subsection{Proof of Theorem \ref{theorem:lower-bound-fixed-consumption-multiple-resource}}
\label{pf:thm_low_det}
Before illustraing the proof, we need a lemma to bound the realized KL divergence, when applying the approach in \cite{CarpentierL16}.
\begin{lemma}
    \label{lemma:bound-realized-KL}
    Denote $X_s\stackrel{i.i.d}{\sim}\mathcal{N}(0, \sigma^2)$. Then for any $N\in \mathbb{N}$, $I>0$, we have
    \begin{align*}
        \Pr\left(\max_{1\leq t\leq N}\sum_{s=1}^t X_s > I\right)\leq \exp\left(-\frac{I^2}{2N\sigma^2}\right)
    \end{align*}
\end{lemma}
To facilitate our discussion, we denote $\mathbb{E}_i[\cdot]$ as the expectation operator corresponding to the probability measure $\Pr_i$. 
Theorem \ref{theorem:lower-bound-fixed-consumption-multiple-resource} is proved in the following two steps. 

\textbf{Step 1.} We show that, under the assumption $\Pr_1(\psi\neq 1) <1/2$, for every $i\in \{2, \ldots, K\}$ it holds that
\begin{equation}\label{eq:det_first}
   \Pr_i(\psi\neq i) \geq \frac{1}{6}\exp\left(-12t_i(\frac{1}{2}-r_i)^2-\sqrt{96t_i(\frac{1}{2}-r_i)^2}\right),
\end{equation}
where 
\begin{equation}\label{eq:tiTi_det}
t_i = \mathbb{E}_1[T_i], \quad T_i = \sum^\tau_{t=1} \mathbf{1}(A(t) = i)
\end{equation}
is the number of times pulling arm $i$. Note that if the assumption $\Pr_1(\psi\neq 1) <1/2$ is violated, the conclusion in Theorem \ref{theorem:lower-bound-fixed-consumption-multiple-resource} immediately holds for $Q^{(1)}$ and sufficiently large enough $\{C_{\ell}\}_{\ell=1}^L$. 

\textbf{Step 2.} We show that there exists $i\in \{2, \ldots K\}$ such that 
\begin{align}
t_i (1/2 - r_i)^2 & \leq \text{min}_{\ell\in [L]} \left\{\frac{2 C_\ell}{H^\text{det}_{2, \ell}(Q^{(1)})}\right\}\label{eq:step_2_det_main}\\
&  \leq \text{min}_{\ell\in [L]} \left\{\frac{2 C_\ell}{H^\text{det}_{2, \ell}(Q^{(i)})}\right\}.\nonumber
\end{align}
This step crucially hinges on the how the consumption model is set in (\ref{eq:consumption_model}). Finally, Theorem \ref{theorem:lower-bound-fixed-consumption-multiple-resource} follows by taking $C_1, \ldots, C_L$
so large that
$$
96\cdot\text{min}_{i\in [K],\ell\in [L]} \left\{\frac{2 C_\ell}{H^\text{det}_{2, \ell}(Q^{(i)})}\right\} \geq 1.
$$
Such $C_1, \ldots, C_L$ exist. For example, we can take $C_1 = \ldots = C_L = C$, then the left hand side of the above condition grows linearly with $C$. Altogether, the Theorem is shown, and it remains to establish \textbf{Steps 1, 2}.

\textbf{Establishing on Step 1.} To establish (\ref{eq:det_first}), we follow the approach in \citep{CarpentierL16} and consider the event  $$\mathcal{E}_i=\{\psi=1\}\cap\{T_i\le 6t_i\}\cap \{\xi\}$$
for $i\in [K]$. The quantities $T_i, t_i$ are as defined in (\ref{eq:tiTi_det}), and $\xi$ is an event concerning an empirical estimate on a certain KL divergence term
\begin{align*}
    \xi=\left\{\max_{1\leq t\leq 6t_i}\sum_{s=1}^{t}\log\frac{d\nu_i^{(1)}}{d\nu_i^{(i)}}(\tilde{R}_i(s))-\frac{(1-2r_i)^2}{2} \leq I\right\}.
\end{align*}
Here $\tilde{R}_i(1), \tilde{R}_i(2),\ldots$ are arm $i$ rewards received during the $T_i$ pulls of arm $i$ in the online dynamics, and $I=\sqrt{24t_i(1-2r_i)^2}$. As we take $\nu_i^{(1)}$ as $\mathcal{N}(r_i, 1)$ and $\nu_i^{(i)}$ as $\mathcal{N}(1-r_i, 1)$, we can conclude
\begin{align*}
    & \log\frac{d\nu_i^{(1)}}{d\nu_i^{(i)}}(\tilde{R}_i(s))\\
    = & \log \frac{\exp(-\frac{(\tilde{R}_i(s)-r_i)^2}{2})}{\exp(-\frac{(\tilde{R}_i(s)-(1-r_i))^2}{2})}\\
    = & \frac{(\tilde{R}_i(s)-(1-r_i))^2}{2} - \frac{(\tilde{R}_i(s)-r_i)^2}{2}\\
    = & \frac{(2r_i-1)(2\tilde{R}_i(s)-1)}{2}.
\end{align*}
If $\tilde{R}_i(s)\sim \nu_i^{(1)}$, we can conclude $\log\frac{d\nu_i^{(1)}}{d\nu_i^{(i)}}(\tilde{R}_i(s))\sim \mathcal{N}(\frac{(1-2r_i)^2}{2}, (1-2r_i)^2)$. By the Lemma \ref{lemma:bound-realized-KL}, we know $\Pr_{\nu_i^{(1)}}(\neg \xi)\leq \exp\left(-\frac{I^2}{12t_i(1-2r_i)^2 }\right)$. Similar as \cite{CarpentierL16}, we can assume $\Pr_{\mathcal{G}^1}(\psi=1)\geq \frac{1}{2}$ always hold, or we have proved the theorem. Then, we can conclude
\begin{align}
    \Pr_{\mathcal{G}^1}(\mathcal{E}_i)\geq & \Pr_{\mathcal{G}^1}(\psi=1) - \Pr_{\mathcal{G}^1}(\neg \xi) - \Pr_{\mathcal{G}^1}(T_i>6t_i)\nonumber\\
    \geq & \frac{1}{2}-\exp\left(-\frac{I^2}{12t_i(1-2r_i)^2 }\right)-\frac{1}{6}\label{eqn:markov}\\
    \geq & \frac{1}{2}-\frac{1}{6}-\frac{1}{6}\label{eqn:e-2<1/3}\\
    = & \frac{1}{6}\nonumber.
\end{align}
Step (\ref{eqn:markov}) is by the Markov Inequality. Step (\ref{eqn:e-2<1/3}) is by the fact that $\frac{1}{e^2}<\frac{1}{6}$.

The event $\xi$ is also considered in \citep{CarpentierL16}, with a different analysis on its probability. Applying the transportation equality as \citep{CarpentierL16}, for every $i\in \{2, \ldots, K\}$, we have

\begin{align}
    & \Pr_i(\psi \neq i)\nonumber\\
    \geq &\Pr_{i}(\mathcal{E}_i)\nonumber\\
    = & \mathbb{E}_{1}\left(\mathbbm{1}\{\mathcal{E}_i\}\exp\left(-\sum_{s=1}^{T_i}\log\frac{d\nu_i^{(1)}}{d\nu_i^{(i)}}(\tilde{R}_i(s))\right)\right)
    \label{eq:by_change_of_measure}\\
    = & \mathbb{E}_{1}\left(\mathbbm{1}\{\mathcal{E}_i\}\exp\left(\sum_{s=1}^{T_i}\frac{(1-2r_i)^2}{2}-\log\frac{d\nu_i^{(1)}}{d\nu_i^{(i)}}(\tilde{R}_i(s))\right)\right)
    \nonumber\\
    & \exp\left(-\frac{T_i(1-2r_i)^2}{2}\right)\nonumber\\
    \geq & \exp\left(-\frac{6t_i(1-2r_i)^2}{2}-I\right)\Pr_{\mathcal{G}^1}(\mathcal{E}_i)\label{eq:by_event_E_i}\\
    \geq & \frac{1}{6}\exp\left(-\frac{6t_i(1-2r_i)^2}{2}-\sqrt{24t_i(1-2r_i)^2}\right)\label{eqn:plug-in-I}\\
    = & \frac{1}{6}\exp\left(-12t_i(\frac{1}{2}-r_i)^2-\sqrt{96t_i(\frac{1}{2}-r_i)^2}\right)\nonumber.
\end{align}
The above calculations establishes \textbf{Step 1}, and we conclude the discussion on \textbf{Step 1} by justifying steps (\ref{eq:by_change_of_measure}-\ref{eqn:plug-in-I}). 

Step (\ref{eq:by_change_of_measure}) is by a change-of-measure identity frequently used in the MAB literature. For example, it is established in equation (6) in \cite{audibert2010best} and Lemma 18 in \cite{KaufmanCG16}. The identity is described as follows. Let $\tau$ be a stopping time with respect to $\{\sigma(H(t))\}^\infty_{t=1}$, where we recall that $H(t)$ is the historical observation up to the end of time step $t$. For any event ${\cal E}\in \sigma(H(\tau))$ and any instance index $i\in \{2, \ldots K\}$, it holds that
\begin{align*}
    \Pr_{\mathcal{G}^i}(\mathcal{E})=\mathbb{E}_{\mathcal{G}^1}\left[\mathbbm{1}\{\mathcal{E}\}\exp\left(-\sum_{s=1}^{T_i}\log\frac{d\nu_i^{(1)}}{d\nu_i^{(i)}}(\tilde{R}_i(s))\right)\right].
\end{align*}
Consequently, step (\ref{eq:by_change_of_measure}) holds by the fact that the choice of arm $\psi$ only depends on the observed trajectory $\sigma(H(\tau))$, and evidently both $T_i$ and $\xi$ are both $\sigma(H(\tau))$-measurable. Step (\ref{eq:by_event_E_i}) is by the event $\mathcal{E}_i$, which asserts $T_i\leq 6t_i$ and $\sum_{s=1}^{T_i}\frac{(1-2r_i)^2}{2}-\log\frac{d\nu_i^{(1)}}{d\nu_i^{(i)}}(\tilde{R}_i(s))\geq -I$. To complete the last step (\ref{eqn:plug-in-I}), we just need to plug in the value of $I$.

\textbf{Establishing Step 2.} To proceed with \textbf{Step 2}, we first define a complexity term $H^{\text{det}}_{1, \ell}(Q)$, which is similar to $H^{\text{det}}_{2, \ell}(Q)$ but the former aids our analysis. 
For a deterministic consumption instance $Q$ (whose arms are not necessarily ordered as $r_1\geq r_2\geq \ldots, r_K$), we denote $\{r_{(k)}\}^K_{k=1}$ as a permutation of $\{r_k\}^K_{k=1}$ such that $r_{(1)} > r_{(2)}\geq   \ldots \geq  r_{(K)}$. For example, when $Q= Q^{(i)}$, we can have $r_{(1)} = r^{(i)}_i$, $r_{(j+1)} = r^{(i)}_j$ for $j\in \{1,\ldots, i-1\}$, and $r_{(j)} = r^{(i)}_j$ for $j\in \{i+1,\ldots, K\}$. Similarly, denote $\{d_{\ell,(k)}\}^K_{k=1}$ as a permutation of $\{d_{\ell,k}\}^K_{k=1}$ such that $d_{\ell,(1)}\geq d_{\ell,(2)}\geq \ldots \geq d_{\ell,(K)}$. 
Define $\Delta_{(1)} = \Delta_{(2)} = r_{(1)} - r_{(2)}$, and define $\Delta_{(k)} = r_{(1)} - r_{(k)}$ for $k\in \{3, \ldots, K\}$. Now, we are ready to define 
\begin{equation}\label{eq:H_1_det}
H^{\text{det}}_{1, \ell}(Q) = \sum^K_{k=1} \frac{d_{\ell, (k)}}{\Delta_{(k)}^{2}}.
\end{equation}

In the special case of $L = 1$ and $d_{1, k} = 1$ for all $k \in [K]$, the quantity $H^{\text{det}}_{1, \ell}(Q)$ is equal to the complexity term $H_1$ defined for BAI in the fixed confidence setting \citep{audibert2010best} (the term $H_1$ is relabeled as $H$ in subsequent research works \cite{KarninKS13, CarpentierL16}). Observe that for any deterministic consumption instance $Q$, we always have 
\begin{equation}\label{eq:observe_H0}
H^{\text{det}}_{2, \ell}(Q) \leq H^{\text{det}}_{1, \ell}(Q).
\end{equation}
In addition, we observe that for any $i\in \{2, \ldots, K\}$ and any $\ell\in [L]$, it holds that
\begin{align}
H^{\text{det}}_{1, \ell}(Q^{(1)})&\ge H^{\text{det}}_{1, \ell}(Q^{(i)}),\label{eq:H_observe1}\\
H^{\text{det}}_{2, \ell}(Q^{(1)})&\ge H^{\text{det}}_{2, \ell}(Q^{(i)}).\label{eq:H_observe2}
\end{align}

After defining $H^{\text{det}}_{1, \ell}(Q)$, we are ready to proceed to establishing \textbf{Step 2}. Recall that $ T_i = \sum^\tau_{t=1} \mathbf{1}(A(t) = i)$ is the number of arm pulls on arm $i$. By the requirement of feasibility and the definition of $d_{\ell, k}$ in (\ref{eq:consumption_model}), we know that
\begin{align*}
    \sum_{i=1}^K\sum_{s=1}^{T_i}D_{i,s,\ell} \leq C_{\ell}
\end{align*}
holds for all $\ell \in [L]$, where $D_{i,s,\ell}$ follows the distribution with mean value $\begin{cases}d_{\ell, (2)} & i=1\\ d_{\ell, (1)} & i=2\\ d_{\ell, (i)} & \text{others}\end{cases}$. From the problem setting, we know $D_{i,s,\ell}$ are independent with the history $\{D_{i,s',\ell}\}_{s'=1}^{s}$. By the Wald's Equality, we can conclude 
\begin{align*}
    t_1d_{\ell,(2)}+ t_2d_{\ell,(1)} + \sum_{k=3}^K t_k d_{\ell,(k)}\le C_{\ell}
\end{align*}

From our definition of $H^\text{det}_{1,\ell}(Q^{(1)})$, for every $\ell\in [L]$ we have
\begin{align*}
    &\frac{d_{\ell, 1}}{H^\text{det}_{1, \ell}(Q^{(1)})(\frac{1}{2}-r_2)^2}+\sum_{k=2}^K \frac{d_{\ell, (k)}}{H^\text{det}_{1, \ell}(Q^{(1)})(\frac{1}{2}-r_k)^2}=1,
\end{align*}
which implies that
\begin{align}
    &\frac{2 C_{\ell} d_{\ell,(1)}}{H^\text{det}_{1, \ell}(Q^{(1)})(\frac{1}{2}-r_2)^2} + \sum_{k=3}^K \frac{C_{\ell} d_{\ell,(k)}}{H^\text{det}_{1, \ell}(Q^{(1)})(\frac{1}{2}-r_k)^2}\nonumber\\
    \ge &t_1d_{\ell,(2)}+t_2d_{\ell,(1)}+t_3d_{\ell,(3)}+\cdots+t_Kd_{\ell,(K)}. \label{eq:crucial_step_2}
\end{align}

holds for any $\ell$. Inequality (\ref{eq:crucial_step_2}) implies that for any $\ell\in [L]$, it is either the case that 
$\frac{2 C_{\ell} \cdot d_{\ell,(1)}}{H^\text{det}_{1, \ell}(Q^{(1)})(\frac{1}{2}-r_2)^2} \geq t_2 d_{\ell, (1)},$
or there exists $k_\ell\in \{3, \ldots, K\}$ such that $\frac{C_{\ell} d_{\ell,(k_\ell)}}{H^\text{det}_{1,\ell}(Q^{(1)})(\frac{1}{2}-r_{k_\ell})^2} \ge t_{k_\ell}d_{\ell,(k_\ell)}$. Collectively, the implication is equivalent to saying that for all $\ell\in [L]$, there exists $k_\ell\in \{ 2, \ldots, K\}$ such that
\begin{align*}
    t_{k_\ell}\left(\frac{1}{2}-r_{k_\ell}\right)^2 \leq \frac{2C_{\ell}}{H^\text{det}_{1,\ell}(Q^{(1)})},
\end{align*}
or more succinctly there exists $i\in \{2, \ldots, K\}$ such that 
$$
 t_{i}\left(\frac{1}{2}-r_{i}\right)^2 \leq \min_{\ell\in [L]}\left\{\frac{2C_{\ell}}{H^\text{det}_{1,\ell}(Q^{(1)})}\right\}.
$$
Finally, \textbf{Step 2} is established by the observations (\ref{eq:observe_H0}, \ref{eq:H_observe1}, \ref{eq:H_observe2}).

\subsection{Improvement of $H_{2, \ell}^{\text{det}}(Q)$ is Unachievable}
\label{sec:Improvement_of_H2_is_unachievable}
For a problem instance $Q$ with mean reward $\{r^Q_k\}_{k=1}^K, r^Q_1\geq r^Q_2\geq\cdots \geq r^Q_K$, mean consumption $\{d^Q_{\ell, k}\}_{k=1, \ell=1}^{K, L}$ and budget $\{C_{\ell}\}_{\ell=1}^L$, we define 
\begin{align}
    \tilde{H}_{1, \ell}^{det}(Q)=&\frac{d_{\ell, 1}}{(r^Q_1-r^Q_2)^2}+\sum_{k=2}^{K}\frac{d_{\ell, k}}{(r^Q_1-r^Q_k)^2}\label{eq:refined_H1_det}\\
    \tilde{H}_{2, \ell}^{det}(Q)=&\max_{2\leq k \leq K} \frac{\sum_{j=1}^k d_{\ell, j}}{(r^Q_1-r^Q_k)^2}\label{eq:refined_H2_det}.
\end{align}
Easy to see $\tilde{H}_{1, \ell}^{det}(Q)\leq H_{1,\ell}^{det}(Q)$, $\tilde{H}_{2, \ell}^{det}(Q)\leq H_{2,\ell}^{det}(Q)$. We want to know whether we can find an algorithm such that for any problem instance $Q$, we can achieve the following upper bound of the failure probability. 
\begin{align}
    & \Pr_Q(failure) \nonumber \\
    \le & poly(K) \exp\left(-\frac{\Omega(1)}{\log_2 K}\min_{\ell\in [L]}\left\{\frac{C_{\ell}}{\tilde{H}_{2, \ell}^{det}(Q)}\right\}\right)\label{eq:refined_upper_det}.
\end{align}


\textbf{The answer is No.} And the analysis method we used is similar to appendix \ref{pf:thm_low_det}. We can construct a list of problem instance $Q^{(i)}$, and prove a lower bound that could be larger than the right side of (\ref{eq:refined_upper_det}), as $K$ and $\{C_{\ell}\}_{\ell=1}^L$ are large enough. 

We focus on $L=1$. Assume there are $C$ units of the resource. Given $K$, let $d_1=\frac{1}{2^{K-2}}, d_k = \frac{1}{2^{K-k}}, k\ge 2$, $r_1=\frac{1}{2}, r_k=\frac{1}{2}-2^{\frac{k-K-4}{2}},k\geq2$. Easy to see $d_1= d_2\leq \cdots \leq d_K$, $\frac{1}{2}=r_1\geq r_2\geq \cdots \geq r_K=\frac{1}{4}$. Then we construct $K$ problem instances $\{Q^{(i)}\}_{i=1}^K$. For problem instance $Q^{(1)}$, the mean reward of $k^{th}$ arm is $r_k$, following the Bernoulli distribution. And the deterministic consumption of $k^{th}$ arm is $d_k$. For problem instance $Q^{(i)}, 2\leq i\leq K$, the mean reward of $k^{th}\ne i$ arm is $r_k$, the mean reward of $i^{th}$ arm is $1-r_i$. And the deterministic consumption of $k^{th}$ arm is $d_k$. For $i\in [K]$, the best arm of $Q^{(i)}$ is always the $i^{th}$ arm. Then we can calculate $\tilde{H}_{1,\ell=1}^{det}(Q^{(i)})$ and $\tilde{H}_{2,\ell=1}^{det}(Q^{(i)})$ for $i\in [K]$. Easy to derive
\begin{align*}
    &\tilde{H}_{1,\ell=1}^{det}(Q^{(1)})\\
    =&\frac{d_1}{(r_1-r_2)^2}+\sum_{k=2}^K \frac{d_k}{(r_1-r_k)^2}\\
    =&\frac{\frac{K}{2}}{2^{K-3}\frac{1}{2^{K+2}}}=16K.
\end{align*}
\begin{align}
    \tilde{H}_{2, \ell=1}^{det}(Q^{(1)})=\max_{k\ge 2} \frac{\sum_{t=1}^k d_{t}}{(r_1-r_k)^2}=32.
\end{align}
For $2\leq i\leq K$, 
\begin{align}
    &\tilde{H}_{1,\ell=1}^{det}(Q^{(i)})\nonumber\\
    =&\frac{d_i+d_1}{(1-r_i-r_1)^2}+\sum_{t=2}^{i-1} \frac{d_t}{(1-r_i-r_t)^2}+\nonumber\\
    &\sum_{t=i+1}^{K} \frac{d_t}{(1-r_i-r_t)^2}\nonumber\\
    =&\frac{\frac{1}{2^{K-i}}+\frac{1}{2^{K-2}}}{(2^{\frac{i-K-4}{2}})^2}+\sum_{t=2}^{i-1}\frac{\frac{1}{2^{K-t}}}{(2^{\frac{i-K-4}{2}}+2^{\frac{t-K-4}{2}})^2}+\nonumber\\
    &\sum_{t=i+1}^{K}\frac{\frac{1}{2^{K-t}}}{(2^{\frac{i-K-4}{2}}+2^{\frac{t-K-4}{2}})^2}\nonumber\\
    =&\frac{2^i+4}{2^{i-4}}+\sum_{t=2}^{i-1}\frac{2^t}{2^{i-4}+2^{\frac{i+t}{2}-3}+2^{t-4}}+\nonumber\\
    &\sum_{t=i+1}^{K}\frac{2^t}{2^{i-4}+2^{\frac{i+t}{2}-3}+2^{t-4}}.
\end{align}
\begin{align}
    &\tilde{H}_{2, \ell=1}^{det}(Q^{(i)})\nonumber\\
    =&\max\{\frac{d_i+d_1}{(1-r_i-r_1)^2}, \max_{2\leq t \leq i-1} \frac{d_i+\sum_{l=1}^{t}d_l}{(1-r_i-r_t)^2},\max_{i+1\leq t\leq K} \frac{\sum_{l=1}^{t}d_l}{(1-r_i-r_t)^2}\}\nonumber\\
    =&\max\{\frac{\frac{1}{2^{K-i}}+\frac{1}{2^{K-2}}}{(2^{\frac{i-K-4}{2}})^2}, \max_{2\leq t \leq i-1}\frac{\frac{1}{2^{K-i}}+\frac{1}{2^{K-t-1}}}{(2^{\frac{i-K-4}{2}}+2^{\frac{t-K-4}{2}})^2},\max_{i+1\leq t\leq K}\frac{\frac{1}{2^{K-t-1}}}{(2^{\frac{i-K-4}{2}}+2^{\frac{t-K-4}{2}})^2}\}\nonumber\\
    =&\max\{\frac{2^i+4}{2^{i-4}}, \max_{2\leq t \leq i-1}\frac{2^i+2^{t+1}}{2^{i-4}+2^{\frac{i+t}{2}-3}+2^{t-4}}, \max_{i+1\leq t\leq K}\frac{2^{t+1}}{2^{i-4}+2^{\frac{i+t}{2}-3}+2^{t-4}}\}.
\end{align}
With a simple calculation, for $2\leq i\leq K$, we have $\frac{2^i+4}{2^{i-4}}\leq 32$, $\frac{2^i+2^{t+1}}{2^{i-4}+2^{\frac{i+t}{2}-3}+2^{t-4}}\leq 32$ and $\frac{2^{t+1}}{2^{i-4}+2^{\frac{i+t}{2}-3}+2^{t-4}}\leq 32$. Thus we can conclude $\tilde{H}_{2, \ell=1}^{det}(Q^{(i)})\leq32=\tilde{H}_{2, \ell=1}^{det}(Q^{(1)})$. On the other hand, easy to check $\tilde{H}_{1,\ell=1}^{det}(Q^{(i)})\leq \tilde{H}_{1,\ell=1}^{det}(Q^{(1)})$ from the definition of $\tilde{H}_{1,\ell=1}^{det}$.

Following the step 1 in appendix \ref{pf:thm_low_det}, we can conclude
for every $i\in \{2, \ldots, K\}$ it holds that
\begin{align}
    \begin{split}
        & \Pr_i(\psi\neq i)\\
        \geq & 
        \frac{1}{6}\exp\left(-12t_i(\frac{1}{2}-r_i)^2-\sqrt{96t_i(\frac{1}{2}-r_i)^2}\right)
    \end{split}
\end{align}
where 
\begin{equation}
t_i = \mathbb{E}_1[T_i], \quad T_i = \sum^\tau_{t=1} \mathbf{1}(A(t) = i)
\end{equation}
is the number of times pulling arm $i$, and 
\begin{equation}
T = \lfloor\frac{C}{d_{1}}\rfloor
\end{equation}
is an upper bound to the number of arm pulls by any policy that satisfies the resource constraints with certainty.

Following the step 2 in appendix \ref{pf:thm_low_det}, recall $d_1=d_2=\frac{1}{2^{K-2}}$, we can derive
\begin{align*}
    \sum_{k=2}^K \frac{2d_k}{\tilde{H}_{1,\ell=1}^{det}(Q^{(1)}) (r_1-r_k)^2}\ge 1.
\end{align*}
Since $\sum_{k=1}^K t_k d_{k}\le C$, we can further conclude
\begin{align*}
    \sum_{k=2}^K \frac{2C d_k}{\tilde{H}_{1,\ell=1}^{det}(Q^{(1)}) (r_1-r_k)^2}\ge \sum_{k=1}^K t_k d_{k}
\end{align*}
which implies there exists $i\ge2$, such that $\frac{2C d_i}{\tilde{H}_{1,\ell=1}^{det}(Q^{(1)}) (r_1-r_i)^2}\ge t_i d_{i}$. For this $i$,
\begin{align*}
    &\mathbb{P}_{\mathcal{G}^i}(\hat{k}\ne i)\\
    \geq & \frac{1}{6}\exp\left(-\frac{24C}{\tilde{H}_{1,\ell=1}^{det}(Q^{(1)})}-\sqrt{\frac{192C}{\tilde{H}_{1,\ell=1}^{det}(Q^{(1)})}}\right).
\end{align*}
When $C$ is large enough, we can assume $\frac{192C}{\tilde{H}_{1,\ell=1}^{det}(Q^{(1)})}>1$, for any bandit strategy that returns the arm $\hat{k}$,
\begin{align*}
    \max_{2\le i\le K} \mathbb{P}_{\mathcal{G}^i}(\hat{k}\ne i)\ge &\frac{1}{6}\exp\left(-216\frac{C}{\tilde{H}_{1,\ell=1}^{det}(Q^{(1)})} \right)\\
    =&\frac{1}{6}\exp\left(-432\frac{C}{K \tilde{H}_{2,\ell=1}^{det}(Q^{(1)})} \right)\\
    \ge& \frac{1}{6}\exp\left(-432\frac{C}{K \tilde{H}_{2,\ell=1}^{det}(Q^{(i)})} \right).
\end{align*}
Now recall the right side of \ref{eq:refined_upper_det}. Assume we can design an algorithm, such that for any problem instance $Q$m the failure probability is upper bounded by $K^{\alpha} \exp\left(-\frac{\beta}{\log_2 K}\frac{C}{\tilde{H}_{2, \ell=1}^{det}(Q)}\right)$ for some positive $\alpha,\beta$ independent with $C, K, \tilde{H}^{\text{det}}_{2,\ell=1}$. Then, we can take $K$ large enough, such that $\frac{432}{K} < \frac{\beta}{\log_2 K}$. In this case, for each $Q^{(i)}$, we have
\begin{align*}
    & \lim_{C\rightarrow +\infty}\frac{\frac{1}{6}\exp\left(-432\frac{C}{K \tilde{H}_{2,\ell=1}^{det}(Q^{(i)})} \right)}{K^{\alpha} \exp\left(-\frac{\beta}{\log_2 K}\frac{C}{\tilde{H}_{2, \ell=1}^{det}(Q^{(1)})}\right)}\\
    = & \lim_{C\rightarrow +\infty} \frac{1}{K^{\alpha}}\exp\left(\Big(\frac{\beta}{\log_2 K}-\frac{432}{K}\Big)\frac{C}{\tilde{H}_{2, \ell=1}^{det}(Q^{(i)})}\right)\\
    = & +\infty.
\end{align*}
Thus, we can find a large enough $C$, such that there exists $i$, we have
\begin{align*}
    \mathbb{P}_{\mathcal{G}^i}(\hat{k}\ne i)\ge &\frac{1}{6}\exp\left(-432\frac{C}{K \tilde{H}_{2,\ell=1}^{det}(Q^{(i)})} \right)\\
    \geq & K^{\alpha} \exp\left(-\frac{\beta}{\log_2 K}\frac{C}{\tilde{H}_{2, \ell=1}^{det}(Q)}\right),
\end{align*}
leading to a contradiction with the lower bound we have proved. That means we should \textbf{never} expect to use the right side of \ref{eq:refined_upper_det} as a general upper bound.

As deterministic consumption setting is a special case of stochastic consumption setting, we also prove the improvement of $H_{2,\ell}(Q)$ is also unachievable.

\subsection{Proof of Theorem \ref{theorem:lower-bound-sto-consumption-multiple-resource}}
\label{pf:thm_low_sto}

Denote $\tau$ as the total number of arm pulls. Before proving theorem theorem \ref{theorem:lower-bound-sto-consumption-multiple-resource} we need to firstly procide a high probability upper bound to $\tau$, with the lemma 2 in \cite{csiszar1998method}.
\begin{lemma}[\cite{csiszar1998method}]
\label{lemma:KL-Divergence-Chernoff-ln}
Denote $\{D_t\}^\infty_{t=1}$ be i.i.d. random variables distributed as $\text{Bern}(d)$, where $d\in (0, 1)$. Let $C$ be a positive real number. Define random variable $\rho=\min\{T: \sum_{t=1}^T D_t \ge C\}$. For any integer $t'\in (C, C / d)$, it holds that
\begin{align*}
    \Pr(\rho \le t')\ge \frac{\exp(-t'\log 2\cdot \text{KL}(C / t', d))}{t'+1},
\end{align*}
where we denote $\text{KL}(p, q) = p\log (p/q) + (1-p)\log((1-p)/(1-q))$.
\end{lemma}
Lemma \ref{lemma:KL-Divergence-Chernoff-ln} suggests with some probability, the pulling times of an arm can not be too large.

Proof of Theorem \ref{theorem:lower-bound-sto-consumption-multiple-resource} is stated as below.

From the assumption $\lim\limits_{d\rightarrow 0^+}\frac{1}{g(d)\log\frac{1}{d}}=+\infty$, we know for any given $i\in \{2,3,\cdots, K\}$, we can find a threshold $\bar{c}$, such that for any $c\in (0, \bar{c})$, $d_{\ell, (j)}= cd^{0}_{\ell, (j)}, \forall j\in [K], \forall \ell\in[L]$, we have 
\begin{align*}
    g(d_{\ell, k})< & \frac{1}{\log\frac{1}{d_{\ell, (i)}}}, \forall \ell\in [L], k\in [K]\\
    \log \frac{1}{1-d_{\ell, (i)}} < & \frac{1}{2},\forall \ell\in [L]\\
    128\left(\frac{g(d_{\ell, (1)})}{(r_1-r_2)^2} + \sum_{k=3}^K \frac{g(d_{\ell, (k)})}{(r_1-r_k)^2}\right)\log\frac{1}{d_{\ell, (i)}} < & 1,\forall \ell\in [L]\\
    C_{\ell} > & \log 64, \forall \ell\in[L]\\
    \log\frac{1}{d_{\ell, (i)}} > & 1, \forall \ell\in[L]
\end{align*}

Given this threshold $\bar{c}$, we complete the proof of Theorem \ref{theorem:lower-bound-sto-consumption-multiple-resource} by the following steps.

\textbf{Step 1.} We show that, 
\begin{align*}
    \Pr_i(\psi\neq i) \geq \Pr_{1}\left((\psi \neq i) \text{ and } (T_i\leq \bar{T}_{i,\ell_0}) \text{ and }\xi_{i,\ell_0}\right)\exp\left(-\frac{3}{4}\frac{C_\ell}{\tilde{H}^{\text{sto}}_{1, \ell_0}}\right).
\end{align*}
holds for all $\ell_0\in [L]$ and $i\in \{2,3,4,\cdots,K\}$,
where $T_i = \sum_{t=1}^{T} \mathbbm{1}(A_t =i)$, $\bar{T}_{i,\ell_0}:=\frac{C_{\ell_0}}{\frac{g(d_{\ell_0,(1)})}{(r_1-r_2)^2} + \sum_{k=3}^K \frac{g(d_{\ell_0,(k)})}{(r_1-r_k)^2}}\frac{1}{(r_1-r_i)^2}=\frac{C_{\ell_0}}{\tilde{H}^{\text{sto}}_{1,\ell_0}}\frac{1}{(r_1-r_i)^2}$, and 
\begin{align*}
    \xi_{i,\ell_0} =\left\{\max_{1\leq t\leq \bar{T}_{i,\ell_0}} \sum_{s=1}^{t}\log\frac{d\nu_i^{(1)}}{d\nu_i^{(i)}}(\tilde{R}_i(s))-\frac{(1-2r_i)^2}{2} \leq I_{i,\ell_0}\right\}
\end{align*}
is the concentration event for the realized KL divergence, similar to the proof of Theorem \ref{theorem:lower-bound-fixed-consumption-multiple-resource}.

\textbf{Step 2.} We assume $\Pr_1\left(\psi = i\right)\leq \exp\left(-\min_{\ell \in [L]}\frac{C_\ell}{\tilde{H}^{\text{sto}}_{1, \ell}}\right)$. If this assumption doesn't hold, we have already proved theorem \ref{theorem:lower-bound-sto-consumption-multiple-resource}. In addition, we will use Lemma \ref{lemma:KL-Divergence-Chernoff-ln} to prove $\Pr_{1}\left(T_i \leq \bar{T}_{i,\ell_0}\right)\geq \exp\left(-\frac{C_{\ell_0}}{\frac{1}{4\log\frac{1}{d_{\ell_0, (i)}}}}\right)$. Combining all the results, we have for all $\ell_0\in [L]$,
\begin{align*}
    & \Pr_{1}\left((\psi \neq i) \text{ and } (T_i \leq \bar{T}_{i,\ell_0}) \text{ and }\xi_{i,\ell_0}\right)\\
    \geq & \Pr_{1}\left(T_i \leq \bar{T}_{i,\ell_0}\right)-\Pr_{1}\left(\neg \xi_{i,\ell_0}\right) - \Pr_{1}(\psi = i)\\
    \geq & \exp\left(-\frac{C_{\ell_0}}{\frac{1}{4\log\frac{1}{d_{\ell_0, (i)}}}}\right) - \exp\left(-\frac{C_{\ell_0}}{32\tilde{H}^{\text{sto}}_{1, \ell_0}}\right) - \exp\left(-\frac{C_{\ell_0}}{\tilde{H}^{\text{sto}}_{1, \ell_0}}\right).
\end{align*}
The last step is also by the fact $\Pr_{1}\left(\neg \xi_{i,\ell_0}\right)\leq \exp\left(-\frac{C_{\ell_0}}{32\tilde{H}^{\text{sto}}_{1, \ell_0}}\right)$, which is guaranteed by the Lemma \ref{lemma:bound-realized-KL}.

\textbf{Step 3.} From the step 1 and step 2, we have concluded
\begin{align*}
    & \Pr_i(\psi\neq i) \\
    \geq & \left(\exp\left(-\frac{C_{\ell_0}}{\frac{1}{4\log\frac{1}{d_{\ell_0, (i)}}}}\right) - \exp\left(-\frac{C_{\ell_0}}{32\tilde{H}^{\text{sto}}_{1, \ell_0}}\right) - \exp\left(-\frac{C_{\ell_0}}{\tilde{H}^{\text{sto}}_{1, \ell_0}}\right)\right)\exp\left(-\frac{3}{4}\frac{C_{\ell_0}}{\tilde{H}^{\text{sto}}_{1, \ell_0}}\right).
\end{align*}
holds for all $\ell_0\in [L]$. Notice that
\begin{align*}
    & 128\left(\frac{g(d_{\ell_0, (1)})}{(r_1-r_2)^2} + \sum_{k=3}^K \frac{g(d_{\ell_0, (k)})}{(r_1-r_k)^2}\right)\log\frac{1}{d_{\ell_0, (i)}} < 1, C_{\ell_0} > \log 64, \log\frac{1}{d_{\ell_0, (i)}} > 1\\
    \Rightarrow & C_{\ell_0} \geq \log 4, \frac{1}{32\tilde{H}^{\text{sto}}_{1, \ell_0}}\geq \frac{1}{\frac{1}{4\log\frac{1}{d_{\ell_0, (i)}}}}\\
    \Leftrightarrow & \log 4 + \frac{C_{\ell_0}}{\frac{1}{4\log\frac{1}{d_{\ell_0, (i)}}}}\leq \frac{C_{\ell_0}}{32\tilde{H}^{\text{sto}}_{1, \ell_0}}\\
    \Leftrightarrow & 4 \exp\left(\frac{C_{\ell_0}}{\frac{1}{4\log\frac{1}{d_{\ell_0, (i)}}}}\right) \leq \exp\left(\frac{C_{\ell_0}}{32\tilde{H}^{\text{sto}}_{1, \ell_0}}\right)\\
    \Rightarrow & \exp\left(-\frac{C_{\ell_0}}{32\tilde{H}^{\text{sto}}_{1, \ell_0}}\right) \leq \frac{1}{4}\exp\left(-\frac{C_{\ell_0}}{\frac{1}{4\log\frac{1}{d_{\ell_0, (i)}}}}\right), \exp\left(-\frac{C_{\ell_0}}{4\tilde{H}^{\text{sto}}_{1, \ell_0}}\right) \leq \frac{1}{2}\exp\left(-\frac{C_{\ell_0}}{\frac{1}{4\log\frac{1}{d_{\ell_0, (i)}}}}\right).
\end{align*}
We can conclude 
\begin{align*}
    & \left(\exp\left(-\frac{C_{\ell_0}}{\frac{1}{4\log\frac{1}{d_{\ell_0, (i)}}}}\right) - \exp\left(-\frac{2C_{\ell_0}}{\tilde{H}^{\text{sto}}_{1, \ell_0}}\right) - \exp\left(-\frac{C_{\ell_0}}{\tilde{H}^{\text{sto}}_{1, \ell_0}}\right)\right)\exp\left(-\frac{3}{4}\frac{C_{\ell_0}}{\tilde{H}^{\text{sto}}_{1, \ell_0}}\right)\\
    \geq & \frac{1}{2}\exp\left(-\frac{C_{\ell_0}}{\frac{1}{4\log\frac{1}{d_{\ell_0, (i)}}}}\right)\exp\left(-\frac{3}{4}\frac{C_{\ell_0}}{\tilde{H}^{\text{sto}}_{1, \ell_0}}\right)\\
    \geq & \exp\left(-\frac{C_{\ell_0}}{\tilde{H}^{\text{sto}}_{1, \ell_0}}\right).
\end{align*}
In summary, we have proved 
\begin{align*}
    \Pr_i(\psi\neq i)\geq \exp\left(-\frac{C_{\ell_0}}{\tilde{H}^{\text{sto}}_{1, \ell_0}}\right).
\end{align*}
The remaining work is to establish step 1 and 2.

\textbf{Establishing on Step 1.} Following section \ref{pf:thm_low_det}, for $i\in \{2,3,\cdots, K\}$, define 
\begin{align*}
    \xi_{i,\ell_0} =\left\{\max_{1\leq t\leq \bar{T}_{i,\ell_0}} \sum_{s=1}^{t}\log\frac{d\nu_i^{(1)}}{d\nu_i^{(i)}}(\tilde{R}_i(s))-\frac{(1-2r_i)^2}{2} \leq I_{i,\ell_0}\right\},
\end{align*}
where $\tilde{R}_i(1),\tilde{R}_i(2),\cdots $ are the arm i rewards received during $\bar{T}_{i,\ell_0}$ pulls, and $I_{i,\ell_0}=\frac{\bar{T}_{i,\ell_0}(1-2r_i)^2}{4}$. With the same calculation in section \ref{pf:thm_low_det}, we know $\tilde{R}_i(s)\sim \nu_i^{(1)}\Rightarrow \log\frac{d\nu_i^{(1)}}{d\nu_i^{(i)}}(\tilde{R}_i(s))\sim \mathcal{N}(\frac{(1-2r_i)^2}{2}, (1-2r_i)^2)$. By the Lemma \ref{lemma:bound-realized-KL}, we can conclude 
\begin{align*}
    \Pr_{\nu_i^{(1)}}(\neg \xi)\leq & \exp\left(-\frac{I^2_{i,\ell_0}}{2\bar{T}_{i,\ell_0}(1-2r_i)^2 }\right)\\
    = & \exp\left(-\frac{\bar{T}_{i,\ell_0}(1-2r_i)^2}{2\cdot 16}\right)\\
    = & \exp\left(-\frac{C_{\ell_0}}{32\tilde{H}^{\text{sto}}_{1, \ell_0}}\right).
\end{align*}
Applying the same transportation equality as \cite{CarpentierL16}, we have
\begin{align}
    & \Pr_i(\psi\neq i)\nonumber\\
    = & \mathbb{E}_1\left(\mathbbm{1}(\psi\neq i)\exp\left(-\sum_{t=1}^{T_i}\log\frac{d\nu_i^{(1)}}{d\nu_i^{(i)}}(\tilde{R}_i(s))\right)\right)\label{eq:sto-trans-equality}\\
    \geq & \mathbb{E}_1\left(\mathbbm{1}(\psi\neq i)\mathbbm{1}(T_i\leq \bar{T}_{i,\ell_0})\mathbbm{1}(\xi_{i,\ell_0}) \exp\left(-\sum_{t=1}^{T_i}\log\frac{d\nu_i^{(1)}}{d\nu_i^{(i)}}(\tilde{R}_i(s))\right)\right)\nonumber\\
    \geq &  \exp\left(-\frac{\bar{T}_{i,\ell_0}(1-2r_i)^2}{2}\right)\mathbb{E}_1\left(\mathbbm{1}(\psi\neq i)\mathbbm{1}(T_i\leq \bar{T}_{i,\ell_0})\mathbbm{1}(\xi_{i,\ell_0}) \exp\left(-\sum_{t=1}^{T_i}\left(\log\frac{d\nu_i^{(1)}}{d\nu_i^{(i)}}(\tilde{R}_i(s))-\frac{(1-2r_i)^2}{2}\right)\right)\right)\label{eq:T_i_upper_sto}\\
    \geq & \exp\left(-\frac{\bar{T}_{i,\ell_0}(1-2r_i)^2}{2}-I\right)\mathbb{E}_1\left(\mathbbm{1}(\psi\neq i)\mathbbm{1}(T_i\leq \bar{T}_{i,\ell_0})\mathbbm{1}(\xi_{i,\ell_0}) \right)\label{eq:sto_xi_kl}\\
    = & \exp\left(-\frac{3\bar{T}_{i,\ell_0}(1-2r_i)^2}{4}\right)\mathbb{E}_1\left(\mathbbm{1}(\psi\neq i)\mathbbm{1}(T_i\leq \bar{T}_{i,\ell_0})\mathbbm{1}(\xi_{i,\ell_0}) \right)\label{eq:sto_I_value}\\
    = & \exp\left(-\frac{3}{4}\frac{C_{\ell_0 }}{\tilde{H}_{1,\ell_0}}\right)\mathbb{E}_1\left(\mathbbm{1}(\psi\neq i)\mathbbm{1}(T_i\leq \bar{T}_{i,\ell_0})\mathbbm{1}(\xi_{i,\ell_0}) \right)
\end{align}
Step (\ref{eq:sto-trans-equality}) is by the transportation equality in \cite{audibert2010best}. Step (\ref{eq:T_i_upper_sto}) is by the event $T_i\leq \bar{T}_{i,\ell_0}$. Step (\ref{eq:sto_xi_kl}) is by the event $\xi_{i,\ell_0}$. Plug in the value of $I_{i,\ell_0}$, we get step (\ref{eq:sto_I_value}) and step (\ref{eq:sto_I_value}) is completed by plugging in the value of $\bar{T}_{i,\ell_0}$.


\textbf{Establishing on Step 2.}
We complete step 2 by proving $\mathbb{P}_1(T_i\leq \bar{T}_{i,\ell_0})\geq \exp\left(-\frac{C_{\ell_0}}{\frac{1}{4\log\frac{1}{d_{\ell_0, (i)}}}}\right)$ holds for all $i\in \{2,3,\cdots,K\}, \ell_0\in [L]$.

Denote $h = \left(\frac{1}{(r_1-r_2)^2} + \sum_{k=3}^K \frac{1}{(r_1-r_k)^2}\right) (r_1-r_i)^2$. Recall in our defined instances $\{Q^{(k)}\}_{k=1}^K$, pulling arm $i\in \{2,3,\cdots,K\}$ always consumes $D\sim Bern(d_{\ell, i})$ unit of resource $\ell$. Then, we can assert for $i\in \{2,3,\cdots,K\}$,
\begin{align*}
    \Pr_1\left(T_i \geq  \bar{T}_{i,\ell_0}\right)
    = & \Pr_1\left(T_i \geq  \bar{T}_{i,\ell_0}, \sum_{s=1}^{T_i}D_{i, \ell_0, s} < C_{\ell_0}, \sum_{s=1}^{\bar{T}_{i,\ell_0}}D_{i,  \ell_0, s} < C_{\ell_0}\right)\\
    \leq & \Pr_1\left(\sum_{s=1}^{\bar{T}_{i,\ell_0}}D_{i, \ell_0, s} < C_{\ell_0}\right)\\
    = & 1-\Pr_1\left(\sum_{s=1}^{\bar{T}_{i,\ell_0}}D_{i, \ell_0, s} \geq C_{\ell_0}\right)\\
    \leq & 1-\Pr_1\left(\sum_{s=1}^{\frac{C_{\ell_0}}{h}\log\frac{1}{d_{\ell_0, (i)}}}D_{i, \ell_0, s} \geq C_{\ell_0}\right).
\end{align*}
where $\{D_{i, \ell_0, s}\}_{s=1}^{+\infty}\stackrel{i.i.d}{\sim} Bern(d_{\ell_0,(i)})$. The last inequality is from the fact that $g(d_{\ell_0, k})<\frac{1}{\log\frac{1}{d_{\ell_0, (i)}}}$ holds for all $k\in [K]$, further $\bar{T}_{i,\ell_0}:= \frac{C_{\ell_0}}{\frac{g(d_{\ell_0,(1)})}{(r_1-r_2)^2} + \sum_{k=3}^K \frac{g(d_{\ell_0,(k)})}{(r_1-r_k)^2}}\frac{1}{(r_1-r_j)^2} > \frac{C_{\ell_0}}{h}\log\frac{1}{d_{\ell_0, (i)}}$. 


Notice that $h\geq 1$, we can conclude $\frac{h}{\log\frac{1}{d_{\ell_0, (i)}}} > \frac{1}{\log \frac{1}{d_{\ell_0, (i)}}} > d_{\ell_0, (i)}$. Then, we can apply lemma \ref{lemma:KL-Divergence-Chernoff-ln} to derive the following
\begin{align}
    & \Pr_1\left(\sum_{s=1}^{\frac{C_{\ell_0}}{h}\log\frac{1}{d_{\ell_0, (i)}}}D_{i, \ell_0,  s} \geq C_{\ell_0}\right)\\
    \geq & \frac{\exp\left(-\frac{C_{\ell_0}}{h}\log\frac{1}{d_{\ell_0, (i)}}\text{KL}\left(\frac{h}{\log\frac{1}{d_{\ell_0, (i)}}}, d_{\ell_0,(i)}\right)\right)}{\frac{C_{\ell_0}}{h}\log\frac{1}{d_{\ell_0, (i)}} +1}\label{eqn:Apply-Lemma-KL}\\
    \geq & \frac{\exp\left(-(\frac{C_{\ell_0}}{h}\log\frac{1}{d_{\ell_0, (i)}})h -\frac{C_{\ell_0}}{h}\log\frac{1}{d_{\ell_0, (i)}}\log\frac{1}{1-d_{\ell_0,(i)}}\right)}{\frac{C_{\ell_0}}{h}\log\frac{1}{d_{\ell_0, (i)}} +1}\label{eqn:KL-Inequality}\\
    \geq & \exp\left(-\frac{C_{\ell_0}}{\frac{1}{\log\frac{1}{d_{\ell_0, (i)}}}} -\frac{C_{\ell_0}}{\frac{1}{\log\frac{1}{d_{\ell_0, (i)}}}}\frac{1}{2} - \frac{C_{\ell_0}}{\frac{1}{\log\frac{1}{d_{\ell_0, (i)}}}}\right)\label{eqn:h>=1-and-e^x>=x+1}\\
    \geq & \exp\left(-\frac{C_{\ell_0}}{\frac{1}{4\log\frac{1}{d_{\ell_0, (i)}}}}\right).
\end{align}
Step (\ref{eqn:Apply-Lemma-KL}) is by lemma \ref{lemma:KL-Divergence-Chernoff-ln}. Step (\ref{eqn:KL-Inequality}) is by the following fact
\begin{align*}
    \text{KL}(\frac{M}{\log\frac{1}{d}}, d) = & \frac{M}{\log\frac{1}{d}} \log\frac{\frac{M}{\log\frac{1}{d}}}{d} + \left(1-\frac{M}{\log\frac{1}{d}}\right)\log\frac{1-\frac{M}{\log\frac{1}{d}}}{1-d}\\
    = & \frac{M}{\log\frac{1}{d}}\log\frac{M}{\log\frac{1}{d}} + M + \left(1-\frac{M}{\log\frac{1}{d}}\right) \log(1-\frac{M}{\log\frac{1}{d}}) + \left(1-\frac{M}{\log\frac{1}{d}}\right) \log\frac{1}{1-d}\\
    \leq & 0 + M + 0 + \left(1-\frac{M}{\log\frac{1}{d}}\right) \log\frac{1}{1-d}\\
    = & M+\log\frac{1}{1-d}
\end{align*}
holds for any $M, d$ such that $\frac{M}{\log\frac{1}{d}}, d\in (0, 1)$. Step (\ref{eqn:h>=1-and-e^x>=x+1}) is due to $h\geq 1$, $\log \frac{1}{1-d_{\ell_0, (i)}} < \frac{1}{2}$ and inequality $e^x\geq x+1$.

Plug in the above inequality, we can conclude $\Pr_1\left(T_i \geq  \bar{T}_{i,\ell_0}\right) \leq 1-\exp\left(-\frac{C_{\ell_0}}{\frac{1}{4\log\frac{1}{d_{\ell_0, (i)}}}}\right)$, suggesting that $\mathbb{P}_1(T_i\leq \bar{T}_{i,\ell_0})\geq \exp\left(-\frac{C_{\ell_0}}{\frac{1}{4\log\frac{1}{d_{\ell_0, (i)}}}}\right)$. We complete step 2.

\subsection{Proof of Lemma \ref{lemma:bound-realized-KL}}
\begin{proof}{Proof of Lemma \ref{lemma:bound-realized-KL}}
Take $\lambda= \frac{I}{N\sigma^2}> 0$. Easy to see $\Pr\left(\max_{1\leq t\leq N}\sum_{s=1}^t X_s > I\right)=\Pr\left(\max_{1\leq t\leq N}\prod_{s=1}^t \exp(\lambda X_s)> \exp(\lambda I)\right) $.

Since $\mathbb{E}\left[\prod_{s=1}^{t+1} \exp(\lambda X_s) | X_1,X_2,\cdots, X_t\right] = \mathbb{E}[\exp(\lambda X_{t+1})]\cdot\prod_{s=1}^{t} \exp(\lambda X_s)\geq \exp(\lambda \mathbb{E}[X_{t+1}])\cdot\prod_{s=1}^{t} \exp(\lambda X_s)=\prod_{s=1}^{t} \exp(\lambda X_s)$, we can conclude $\{\prod_{s=1}^{t} \exp(\lambda X_s)\}_{t=1}^{+\infty}$ is a submartingale. By the Theorem 3.10 in \cite{LattimoreS2020}, we have
\begin{align*}
    \Pr\left(\max_{1\leq t\leq N}\sum_{s=1}^t X_s > I\right)= & \Pr\left(\max_{1\leq t\leq N}\prod_{s=1}^t \exp(\lambda X_s)> \exp(\lambda I)\right)\\
    \leq & \frac{\mathbb{E} \prod_{s=1}^N \exp(\lambda X_s)  }{\exp(\lambda I)}\\
    = & \frac{\exp\left(\frac{N\sigma^2\lambda^2}{2}\right) }{\exp(\lambda I)}\\
    = & \exp\left(-\frac{I^2}{2N\sigma^2}\right).
\end{align*}
\end{proof}

\section{Details on the Numerical Experiment Set-ups}
\label{sec:Details-on-the-Numerical-Experiment-Set-ups}
In what follows, we provide details about how the numerical experiments are run. All the numerical experiments were run on the Kaggle servers. 
\subsection{Single Resource, i.e, $L=1$}
\label{sec:Details-on-the-Numerical-Experiment-Set-ups-single}
The details about Figure \ref{fig:numeric-result} is as follows. Firstly, the bars in the plot are more detailedly explained as follows: From left to right, the 1st blue column is matching high reward and high consumption, considering deterministic resource consumption. The 2nd orange column is matching high reward and low consumption, considering deterministic consumption. The 3rd green column is matching high reward and high consumption, considering correlated reward and consumption. The 4th red column is matching high reward and low consumption, considering correlated reward and consumption. The 5th purple column is matching high reward and high consumption, considering uncorrelated reward and consumption. The 6th brown column is matching high reward and low consumption, considering uncorrelated reward and consumption.

Next, we list down the detailed about the setup.
\begin{enumerate}
    \item One group of suboptimal arms, High match High
    
    $r_1=0.9$; $r_i=0.8, i=2,\cdots,256$; $d_{\ell=1, i}=0.9, i=1,\cdots,128$; $d_{\ell=1, i}=0.1, i=129,\cdots,256$
    
    \item One group of suboptimal arms, High match Low
    
    $r_1=0.9$; $r_i=0.8, i=2,\cdots,256$; $d_{\ell=1, i}=0.1, i=1,\cdots,128$; $d_{\ell=1, i}=0.9, i=129,\cdots,256$
    
    \item Trap, High match High
    
    $r_1=0.9$; $r_i=0.8, i=2,\cdots,32$; $r_i=0.1, i=33,\cdots,256$; $d_{\ell=1, i}=0.9, i=1,\cdots,128$; $d_{\ell=1, i}=0.1, i=129,\cdots,256$
    
    \item Trap, High match Low
    
    $r_1=0.9$; $r_i=0.8, i=2,\cdots,32$; $r_i=0.1, i=33,\cdots,256$; $d_{\ell=1, i}=0.1, i=1,\cdots,128$; $d_{\ell=1, i}=0.9, i=129,\cdots,256$
    
    \item Polynomial, High match High
    
    $r_1=0.9, r_i=0.9(1-\sqrt{\frac{i}{256}}), i\geq 2$. $d_{\ell=1, i}=0.9, i=1,\cdots,128$; $d_{\ell=1, i}=0.1, i=129,\cdots,256$
    
    \item Polynomial, High match Low
    
    $r_1=0.9, r_i=0.9(1-\sqrt{\frac{i}{256}}), i\geq 2$. $d_{\ell=1, i}=0.1, i=1,\cdots,128$; $d_{\ell=1, i}=0.9, i=129,\cdots,256$
    
    \item Geometric, High match High
    
    $r_1=0.9, r_{256}=0.1$, $\{r_i\}_{i=1}^{256}$ is geometric, $r_i = 0.9 * (\frac{1}{9})^{\frac{i-1}{255}}$. $d_{\ell=1, i}=0.9, i=1,\cdots,128$; $d_{\ell=1, i}=0.1, i=129,\cdots,256$
    
    \item Geometric, High match Low
    
    $r_1=0.9, r_{256}=0.1$, $\{r_i\}_{i=1}^{256}$ is geometric, $r_i = 0.9 * (\frac{1}{9})^{\frac{i-1}{255}}$. $d_{\ell=1, i}=0.1, i=1,\cdots,128$; $d_{\ell=1, i}=0.9, i=129,\cdots,256$
\end{enumerate}
There are three kinds of consumption.
\begin{enumerate}
    \item Deterministic Consumption. The consumption of each arm are deterministic.
    \item Uncorrelated Consumption. When we pull an arm, the consumption and reward follow Bernoulli Distribution and are independent.
    \item Correlated Consumption. When we pull the arm $i$, the consumption is $D_{\ell=1, i}=\mathbbm{1}(U\le d_{\ell=1,i})$, $D_{\ell=2, i}=\mathbbm{1}(U\le d_{\ell=2,i})$, $R=\mathbbm{1}(U\le r_i)$, where $U$ follows uniform distribution on $[0,1]$
\end{enumerate}

\subsection{Multiple Resources}
\label{sec:Details-on-the-Numerical-Experiment-Set-ups-multiple}
Similarly, in multiple resources cases, we still considered different setups of mean reward, consumption, and consumption setups. 
\begin{enumerate}
    \item One group of suboptimal arms, High match High
    
    $r_1=0.9$; $r_i=0.8, i=2,\cdots,256$; $d_{\ell=1, i}=0.9, i=1,\cdots,128$; $d_{\ell=1, i}=0.1, i=129,\cdots,256$; $d_{\ell=2, i}=0.9, i=1,\cdots,128$; $d_{\ell=2, i}=0.1, i=129,\cdots,256$
    
    \item One group of suboptimal arms, Mixture
    
    $r_1=0.9$; $r_i=0.8, i=2,\cdots,256$; $d_{\ell=1, i}=0.1, i=1,\cdots,128$; $d_{\ell=1, i}=0.9, i=129,\cdots,256$; $d_{\ell=2, i}=0.9, i=1,\cdots,128$; $d_{\ell=2, i}=0.1, i=129,\cdots,256$
    
    \item One group of suboptimal arms, High match Low
    
    $r_1=0.9$; $r_i=0.8, i=2,\cdots,256$; $d_{\ell=1, i}=0.1, i=1,\cdots,128$; $d_{\ell=1, i}=0.9, i=129,\cdots,256$; $d_{\ell=2, i}=0.1, i=1,\cdots,128$; $d_{\ell=2, i}=0.9, i=129,\cdots,256$
    
    \item Trap, High match High
    
    $r_1=0.9$; $r_i=0.8, i=2,\cdots,32$; $r_i=0.1, i=33,\cdots,256$; $d_{\ell=1, i}=0.9, i=1,\cdots,128$; $d_{\ell=1, i}=0.1, i=129,\cdots,256$; $d_{\ell=2, i}=0.9, i=1,\cdots,128$; $d_{\ell=2, i}=0.1, i=129,\cdots,256$;
    
    \item Trap, Mixture
    
    $r_1=0.9$; $r_i=0.8, i=2,\cdots,32$; $r_i=0.1, i=33,\cdots,256$; $d_{\ell=1, i}=0.1, i=1,\cdots,128$; $d_{\ell=1, i}=0.9, i=129,\cdots,256$; $d_{\ell=2, i}=0.9, i=1,\cdots,128$; $d_{\ell=2, i}=0.1, i=129,\cdots,256$;
    
    \item Trap, High match Low
    
    $r_1=0.9$; $r_i=0.8, i=2,\cdots,32$; $r_i=0.1, i=33,\cdots,256$; $d_{\ell=1, i}=0.1, i=1,\cdots,128$; $d_{\ell=1, i}=0.9, i=129,\cdots,256$; $d_{\ell=2, i}=0.1, i=1,\cdots,128$; $d_{\ell=2, i}=0.9, i=129,\cdots,256$
    
    \item Polynomial, High match High
    
    $r_1=0.9, r_i=0.9(1-\sqrt{\frac{i}{256}}), i\geq 2$. $d_{\ell=1, i}=0.9, i=1,\cdots,128$; $d_{\ell=1, i}=0.1, i=129,\cdots,256$; $d_{\ell=2, i}=0.9, i=1,\cdots,128$; $d_{\ell=2, i}=0.1, i=129,\cdots,256$
    
    \item Polynomial, Mixture
    
    $r_1=0.9, r_i=0.9(1-\sqrt{\frac{i}{256}}), i\geq 2$. $d_{\ell=1, i}=0.1, i=1,\cdots,128$; $d_{\ell=1, i}=0.9, i=129,\cdots,256$; $d_{\ell=2, i}=0.9, i=1,\cdots,128$; $d_{\ell=2, i}=0.1, i=129,\cdots,256$
    
    \item Polynomial, High match Low
    
    $r_1=0.9, r_i=0.9(1-\sqrt{\frac{i}{256}}), i\geq 2$. $d_{\ell=1, i}=0.1, i=1,\cdots,128$; $d_{\ell=1, i}=0.9, i=129,\cdots,256$; $d_{\ell=2, i}=0.1, i=1,\cdots,128$; $d_{\ell=2, i}=0.9, i=129,\cdots,256$
    
    \item Geometric, High match High
    
    $r_1=0.9, r_{256}=0.1$, $\{r_i\}_{i=1}^{256}$ is geometric, $r_i = 0.9 * (\frac{1}{9})^{\frac{i-1}{255}}$. $d_{\ell=1, i}=0.9, i=1,\cdots,128$; $d_{\ell=1, i}=0.1, i=129,\cdots,256$; $d_{\ell=2, i}=0.9, i=1,\cdots,128$; $d_{\ell=2, i}=0.1, i=129,\cdots,256$;
    
    \item Geometric, Mixture
    
    $r_1=0.9, r_{256}=0.1$, $\{r_i\}_{i=1}^{256}$ is geometric, $r_i = 0.9 * (\frac{1}{9})^{\frac{i-1}{255}}$. $d_{\ell=1, i}=0.1, i=1,\cdots,128$; $d_{\ell=1, i}=0.9, i=129,\cdots,256$; $d_{\ell=2, i}=0.9, i=1,\cdots,128$; $d_{\ell=2, i}=0.1, i=129,\cdots,256$;
    
    \item Geometric, High match Low
    
    $r_1=0.9, r_{256}=0.1$, $\{r_i\}_{i=1}^{256}$ is geometric, $r_i = 0.9 * (\frac{1}{9})^{\frac{i-1}{255}}$. $d_{\ell=1, i}=0.1, i=1,\cdots,128$; $d_{\ell=1, i}=0.9, i=129,\cdots,256$; $d_{\ell=2, i}=0.1, i=1,\cdots,128$; $d_{\ell=2, i}=0.9, i=129,\cdots,256$
\end{enumerate}
There are two kinds of consumption.
\begin{enumerate}
    \item Uncorrelated Consumption. When we pull an arm, the consumption and reward follow Bernoulli Distribution and are independent.
    \item Correlated Consumption. When we pull the arm $i$, the consumption is $D_{\ell=1, i}=\mathbbm{1}(U\le d_{\ell=1,i})$, $D_{\ell=2, i}=\mathbbm{1}(U\le d_{\ell=2,i})$, $R=\mathbbm{1}(U\le r_i)$, where $U$ follows uniform distribution on $[0,1]$
\end{enumerate}

\subsection{Detailed Setting of Real-World Dataset}
\label{sec:details-realworld}
We adopted K Nearest Neighbour, Logistic Regression, Random Forest, and Adaboost as our candidates for the classifiers. And we applied each combination of machine learning model and its hyper-parameter to each supervised learning task with 500 independent trials. We identified the combination with the lowest empirical mean cross-entropy as the best arm. 

Our BAI experiments were conducted across 100 independent trials. During each arm pull in a BAI experiment round—i.e., selecting a machine learning model with a specific hyperparameter combination—we partitioned the datasets randomly into training and testing subsets, maintaining a testing fraction of 0.3. The training subset was utilized to train the machine learning models, and the cross-entropy computed on the testing subset served as the realized reward. We flattened the 2-D image as a vector if the dataset is consists of images. All the experiments are deploied on the Kaggle Server with default CPU specifications. 

The details of the real-world datasets we used are as follows
\begin{itemize}
    \item To classify labels 3 and 8 in part of the MNIST Dataset. (MNIST 3\&8)
    
    Number of label 3: 1086, Number of label 8: 1017, Number of Atrributes: 784.

    Time budget: 60 seconds.

    Link of dataset: \url{https://www.kaggle.com/competitions/digit-recognizer}.
    \item Optical Recognition of Handwritten Digits Data Set. (Handwritten)

    Number of Instances: 3823, Number of Attributes: 64

    Time budget: 60 seconds.

    Link of dataset: \url{https://archive.ics.uci.edu/ml/datasets/optical+recognition+of+handwritten+digits}
    \item To classify labels -1 and 1 in the MADELON dataset. (MADELON)

    Number of Instances: 2000, Number of Attributes: 500.

    Time budget: 80 seconds.

    Link of dataset: \url{https://archive.ics.uci.edu/ml/datasets/Madelon}. 
    \item To classify labels -1 and 1 in the Arcene dataset. (Arcene)
    
    Number of Instances: 200, Number of Attributes:10000

    Time budget: 150 seconds.

    Link of dataset: \url{https://archive.ics.uci.edu/ml/datasets/Arcene} (Arcene)
    \item To classify labels of weight conditions in the Obesity dataset. (Obesity)

    Number of Instances: 2111, Number of Attributes: 16.

    Time budget: 20 seconds.
    
    Link of dataset:
    \url{https://archive.ics.uci.edu/dataset/544/estimation+of+obesity+levels+based+on+eating+habits+and+physical+condition}.
\end{itemize}

The machine learning models and candidate hyperparameters, aka arms, are as follows. We implemented all these models through the scikit-learn package in https://scikit-learn.org/stable/index.html
\begin{itemize}
    \item K Nearest Neighbour
    \begin{itemize}
        \item n\_neighbours = 5, 15, 25, 35, 45, 55, 65, 75
    \end{itemize}
    \item Logistic Regression
    \begin{itemize}
        \item Regularization = "l2" or None
        \item Intercept exists or not exists
        \item Inverse value of regularization coefficient= 1, 2
    \end{itemize}
    \item Random Forest
    \begin{itemize}
        \item Fix max\_depth = 5
        \item n\_estimators = 10, 20, 30, 50
        \item criterion = "gini" or "entropy"
    \end{itemize}
    \item Adaboost
    \begin{itemize}
        \item n\_estimators = 10, 20, 30, 40
        \item learning rate = 1.0, 0.1
    \end{itemize}
\end{itemize}

\end{document}